\documentclass[journal]{IEEEtran}

\usepackage{cite,todonotes}
\usepackage{float}
\usepackage{graphicx}
\usepackage{hyperref}
\usepackage{algpseudocode}
\usepackage{algorithm}
\usepackage{amsthm}
\usepackage{cancel}

\hypersetup{backref=true,       
                    pagebackref=true,               
                    hyperindex=true,                
                    colorlinks=true,                
                    breaklinks=true,                
                    urlcolor= black,                
                    bookmarks=true,                 
                    bookmarksopen=false,
                    linkcolor=red,
                    filecolor=black,
                    citecolor=blue,
}
\usepackage[
	n,
	operators,
	advantage,
	sets,
	adversary,
	landau,
	probability,
	notions,	
	logic,
	ff,
	mm,
	primitives,
	events,
	complexity,
	asymptotics,
	keys]{cryptocode}
\usepackage{amsmath,amssymb,amsfonts}
\usepackage{graphicx}
\usepackage{textcomp}
\usepackage{xcolor,color}
\usepackage{xcolor,colortbl}
\definecolor{Gray}{gray}{0.9}
\usepackage{algorithm}
\usepackage{multirow}
\usepackage{expl3}
\usepackage{subcaption}
\usepackage{footnote}
\usepackage{csquotes}
\usepackage{todonotes}
\usepackage{caption}

\def\BibTeX{{\rm B\kern-.05em{\sc i\kern-.025em b}\kern-.08em
    T\kern-.1667em\lower.7ex\hbox{E}\kern-.125emX}}

    \newtheorem{theorem}{Theorem}

   \newtheorem{definition}{Definition}

\algnewcommand\algorithmicforeach{\textbf{for each}}
\algdef{S}[FOR]{ForEach}[1]{\algorithmicforeach\ #1\ \algorithmicdo}   
\newtheorem{Proposition}{Proposition}

\newtheorem{Observation}{Observation}
\begin{document}
\title{Wormhole Dynamics in Deep Neural Networks}

\author{Yen-Lung Lai\textsuperscript{},        
        Zhe Jin\textsuperscript{*}
      
%% \thanks{Corresponding author: Susanto Rahardja}

 \thanks{\textsuperscript{*}Corresponding author: Zhe Jin is with Anhui Provincial Key Laboratory of Secure Artificial Intelligence, Anhui Provincial International Joint Research Center for Advanced Technology in Medical Imaging, School of Artificial Intelligence, Anhui University, Hefei 230093, China. (email:  jinzhe@ahu.edu.cn)}

 \thanks{\textsuperscript{}Yen-Lung Lai was in Anhui Provincial Key Laboratory of Secure Artificial Intelligence, Anhui Provincial International Joint Research Center for Advanced Technology in Medical Imaging, School of Artificial Intelligence, Anhui University, Hefei 230093, China. He is now with the LKCFES, Department of Computing, Universiti Tunku Abdul Rahman, Malaysia. (email: laiyl@utar.edu.my)}
%% \thanks{Yiqi Tew is with the Faculty of Computing and Information Technology, Tunku Abdul Rahman University College, Malaysia. (e-mail: yiqi@tarc.edu.my )}
\thanks{This work was supported by the National Natural Science Foundation of China  (No. 62376003) and Anhui Provincial Natural Science Foundation (No. 2308085MF200).
}

}

% \author{Yen-Lung Lai,~\IEEEmembership{Student Member,~IEEE,}
%         Zhe Jin,~\IEEEmembership{SM,~IEEE,}
%         KokShiek Wong,~\IEEEmembership{SM,~IEEE,}
%         and~Massimo Tistarelli,~\IEEEmembership{SM,~IEEE}
% % \thanks{Corresponding author: Susanto Rahardja}
% \thanks{LieLin Pang is with the Faculty of Computer science and Information Technology, University of Malaya, Malaysia. (e-mail: adpangll@siswa.um.edu.my)}
% \thanks{KokSheik Wong is with the School of Information Technology, Monash University Malaysia, Malaysia. (email: wong.koksheik@monash.edu)}
% \thanks{Yiqi Tew is with the Faculty of Computing and Information Technology, Tunku Abdul Rahman University College, Malaysia. (e-mail: yiqi@tarc.edu.my )}
% \thanks{Susanto Rahardja is with the Northwestern Polytechnical University, China. (e-mail: susanto@nwpu.edu.cn)}
% }

\maketitle
\thispagestyle{plain}
\pagestyle{plain}

\begin{abstract}
This work investigates the generalization behavior of deep neural networks (DNNs), focusing on the phenomenon of "fooling examples," where DNNs confidently classify inputs that appear random or unstructured to humans. To explore this phenomenon, we introduce an analytical framework based on maximum likelihood estimation, without adhering to conventional numerical approaches that rely on gradient-based optimization and explicit labels. Our analysis reveals that DNNs operating in an overparameterized regime exhibit a collapse in the output feature space. While this collapse improves network generalization, adding more layers eventually leads to a state of degeneracy, where the model learns trivial solutions by mapping distinct inputs to the same output, resulting in zero loss. Further investigation demonstrates that this degeneracy can be bypassed using our newly derived "wormhole" solution. The wormhole solution, when applied to arbitrary fooling examples, reconciles meaningful labels with random ones and provides a novel perspective on shortcut learning. These findings offer deeper insights into DNN generalization and highlight directions for future research on learning dynamics in unsupervised settings to bridge the gap between theory and practice.
\end{abstract}

\section{Introduction}

The rapid advancement of deep neural networks (DNNs) has revolutionized industries, enabling precise decision-making across healthcare, finance, national security, and infrastructure. Applications range from medical diagnosis and financial forecasting to ethical AI governance, showcasing their unparalleled ability to tackle complex, data-intensive problems \cite{spears2018deep, rudin2019stop, zheng2019finbrain, franzke2021data}. Despite these successes, the internal mechanics of DNNs remain poorly understood. This gap between theoretical understanding and practical success is particularly concerning, given the profound societal impacts of these networks.

One of the most puzzling aspects of DNNs is their ability to generalize effectively, even in highly overparameterized regimes where the number of model parameters far exceeds the number of training samples \cite{zhang2021understanding}. Classical statistical theory predicts severe overfitting in such cases, yet empirical evidence often shows the opposite: DNNs frequently exhibit robust generalization without explicit regularization, even when overparameterized.

Additionally, the training of DNNs is marked by significant challenges due to their highly non-convex loss landscapes, characterized by numerous local minima, low-curvature regions arising from saturating nonlinearities, and issues such as gradient explosion and vanishing gradients during backpropagation \cite{keskar2016large, dauphin2013big, bengio1994learning}. Despite these optimization complexities \cite{blum1988training}, gradient-based methods frequently succeed in locating global minimizers, achieving zero or near-zero training loss even when data and labels are randomized before training \cite{zhang2021understanding}. These observations give rise to two fundamental questions that are pivotal for connecting theoretical insights with the practical success of DNNs:
\begin{enumerate}
    \item Why does training deep neural networks often appear to converge successfully to a global minimum, despite the theoretical non-convexity of the loss landscape?
    \item Why do DNNs generalize well in practice despite their capacity for overfitting in an overparameterized regime?
\end{enumerate}
The second question is particularly important, as it directly determines the practical performance of DNNs and their ability to generalize across various data distributions. However, it is equally important to recognize that understanding the success of deep learning also requires attention to factors inherent in the training process, as outlined by the first question.

While previous research efforts have attempted to explain DNN convergence, the highly nonlinear architecture of typical DNNs complicates theoretical analysis. As a result, recent research has focused on deep linear networks \cite{arora2018convergence, bah2022learning, nguegnang2021convergence, saxe2013exact, arora2018optimization, arora2019implicit, chou2024gradient}. Unlike their nonlinear counterparts, which employ activation functions for greater expressivity, deep linear networks often struggle with real-world tasks, leading to underfitting and suboptimal performance. Nonetheless, deep linear networks exhibit intriguing nonlinear learning dynamics as their depth increases, making the exploration of gradient descent convergence properties a non-trivial task. Thus, addressing the aforementioned questions remains challenging even for linear neural networks due to their inherent non-convexity.

Existing literature often attributes the success of DNNs to implicit biases introduced by optimization algorithms such as Stochastic Gradient Descent (SGD) \cite{neyshabur2017exploring, pesme2021implicit, frei2022implicit}. These biases guide DNNs toward solutions that generalize well, as SGD tends to select simpler solutions and mitigate overfitting \cite{arpit2017closer, kalimeris2019sgd, valle2018deep}. Recent studies have also highlighted the phenomenon of shortcut learning, wherein DNNs identify superficial features that serve as shortcuts for solving classification tasks \cite{geirhos2020shortcut, hermann2020shapes, scimeca2021shortcut}. Despite the extensive focus on supervised learning, which demands substantial labeled training data, the investigation of implicit biases in unsupervised learning remains largely underexplored.

In this study, we examine the capacity of DNNs to function effectively in unsupervised settings, particularly when confronted with fooling examples—inputs that resemble random noise to human perception. We demonstrate that overparameterized DNNs can leverage maximum likelihood estimation to extract meaningful representations from random, unlabeled data. Specifically, our analysis reveals that overparameterized feedforward neural networks exhibit the phenomenon of Neural Collapse \cite{papyan2020prevalence}, characterized by input samples organizing into distinct feature clusters with minimal intra-cluster variability and maximal inter-cluster separation. These findings provide a theoretical foundation for understanding how DNNs achieve contrastive learning and generalization, even when processing inputs that appear unstructured or random.

\section{Related Works}

\subsection{Implicit Biases in SGD}

Extensive research has explored the reasons behind the success of deep learning, but these explanations have come under increasing scrutiny. One key area of focus is the concept of implicit bias in DNN training \cite{neyshabur2017exploring}. In overparameterized models, multiple local minima can minimize training error, but not all of them generalize well. Thus, minimizing training error alone is insufficient for effective learning; selecting the wrong minimum can result in poor generalization. It is hypothesized that generalization behavior depends on the algorithm used to minimize training error. Since Stochastic Gradient Descent (SGD) is the predominant optimization algorithm for DNN training, it has become the central focus for examining implicit bias through gradient flow analysis. This involves studying how the iterative updates performed by the algorithm shape the learning dynamics and generalization behavior of neural networks \cite{pesme2021implicit, frei2022implicit}.

Significant contributions from researchers such as Arora et al. \cite{arora2018convergence}, Bah et al. \cite{bah2022learning}, and Nguegnang et al. \cite{nguegnang2021convergence} emphasize the necessity of specific conditions for DNNs trained with SGD to converge to a global minimum.

However, understanding the role of implicit bias in SGD as a factor contributing to the success of DNNs remains a complex challenge. Geiping et al. \cite{geiping2021stochastic} argue that the implicit bias of SGD can be matched or replaced by explicit regularization. Their experiments showed that full-batch training achieved performance comparable to an optimized SGD baseline, reaching an accuracy of 95.67\% for a ResNet-18 model on the CIFAR-10 dataset. These results suggest that implicit bias alone may not fully account for the effectiveness of SGD in training DNNs.

Departing from the gradient flow focus of prior studies, Chiang et al. \cite{chiang2022loss} offered a novel perspective by training neural networks using various gradient-free optimizers. Experimental results showed that these gradient-free optimizers could achieve test accuracies comparable to SGD, challenging the prevailing belief that effective generalization in DNNs is solely due to the implicit biases induced by SGD.

\subsection{Simplicity Biases in SGD}

As research has advanced, increasing evidence points to simplicity biases in SGD. While these biases share similarities with implicit biases attributed to SGD, recent studies have revealed refined insights into how SGD exhibits simplicity biases, guiding DNNs to prioritize simple features. This preference for simpler features helps DNNs avoid overfitting, ultimately improving generalization \cite{arpit2017closer}.

Kalimeris et al. \cite{kalimeris2019sgd} found that in the early stages of training, SGD's performance improvements can be linked to learning a linear classifier. As training continues, SGD begins to learn more complex functions. They demonstrated this in an overparameterized linear model trained using SGD, which perfectly fit the training data while achieving optimal population accuracy when initialized with simple configurations. Their results suggest that a linear classifier can effectively explain the predictions of a highly nonlinear neural network in the initial phases of training.

Valle et al. \cite{valle2018deep} extended this understanding by exploring the parameter-function map, which describes the relationship between a model's parameters and the functions it can express. By associating specific parameter configurations with certain functions, this concept offers insights into how optimization algorithms like SGD navigate the parameter space and impact DNN behavior in function space, particularly with regard to generalization. Their research showed that functions with higher learning probabilities tend to have lower complexity, measured by metrics such as Lempel-Ziv complexity, suggesting that DNNs have an inherent bias toward simpler functions during training.

Supporting this, Shah et al. \cite{shah2020pitfalls} conducted experiments on the LSN (linear, 3-slab, and noise) dataset, which contains a linear predictive coordinate and a more complex 3-slab coordinate, with noise coordinates as irrelevant features. Despite theoretical expectations for equal treatment of the linear and 3-slab coordinates, DNNs trained with SGD prioritized the simpler linear coordinate over the more complex 3-slab coordinate. This highlighted a preference for simplicity over margin in the learning process. Surprisingly, the simplicity bias persisted even when the linear coordinate had a smaller margin than the 3-slab coordinate. Shah et al. also observed that simplicity biases in SGD can sometimes undermine generalization on more challenging or noisy tasks.

\subsection{Shortcut Learning}

The simplicity bias in DNNs appears to be non-arbitrary. Recent work by Geirhos et al. \cite{geirhos2020shortcut} connects this bias to shortcut learning, where models rely on superficial characteristics of the dataset instead of capturing deeper, meaningful structures. Such reliance can result in erroneous predictions, as the model fails to generalize beyond the surface-level patterns in the dataset. This shortcut learning may explain Shah et al.'s findings \cite{shah2020pitfalls}, where simplicity biases relies on 'shortcuts' that can lead to poor generalization.

Hermann et al. \cite{hermann2020shapes} investigated the use of features of varying difficulty, defining them based on the minimum network complexity required for feature extraction. They found that models prefer "easy" linear features over "difficult" non-linear features, even if the latter have higher predictive power. In a subsequent study, Hermann et al. \cite{hermann2023foundations} introduced the concept of "availability," extending the notion of predictivity to include factors affecting a model's likelihood of using a feature. Their controlled experiments showed that non-linear models, particularly deeper ones, exhibited a stronger shortcut bias. Theoretical analysis confirmed that availability bias is inevitable in non-linear architectures like ReLU networks, contrasting with the behavior of unbiased linear models.

Scimeca et al. \cite{scimeca2021shortcut} emphasized that some features are preferred over others, even when equally accessible for learning. They attributed this tendency to the Kolmogorov complexity of the features, noting that low-complexity (Kolmogorov-simple) features dominate the parameter space. As a result, DNNs exhibit a preference for simpler features during training, reinforcing the simplicity bias.

Rahaman et al. \cite{rahaman2019spectral} used Fourier analysis to uncover a significant learning bias in DNNs toward low-frequency functions, termed ‘spectral bias.’ Despite the theoretical ability of DNNs to approximate a broad range of functions, they tend to prioritize learning lower-frequency modes early in training, influencing the final parameterization.

Teney et al. \cite{teney2024neural} offered a new perspective, arguing against the prevailing view that simplicity bias arises from the architecture or gradient descent dynamics. They demonstrated that DNNs inherently favor functions of specific complexity, as measured by Fourier frequency, polynomial order, and compressibility. This bias stems from the network's parameterization itself rather than the training process. Their findings indicate that even networks initialized with random weights exhibit simplicity bias due to the properties of activation functions like ReLU, suggesting that simplicity bias is an intrinsic feature of DNNs, embedded in their design rather than a byproduct of training algorithms.

\section{Addressing the Paradox: Fooling Example}

\begin{figure}[!htp]
\centering
  \includegraphics[scale=0.54]{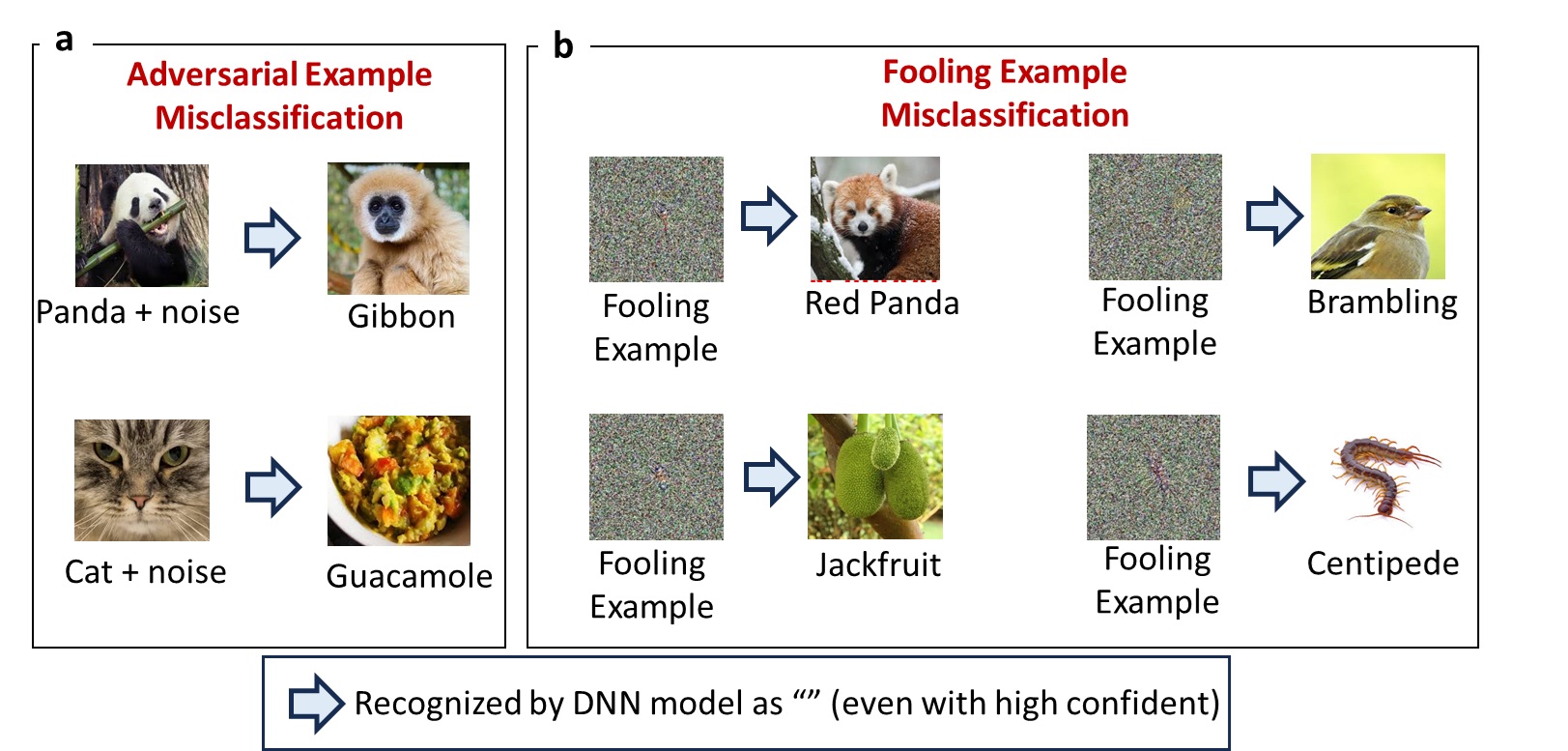}
  \caption{Illustrative examples showcasing DNN misclassifications caused by adversarial and fooling examples.}
  \label{fig:kayleedemo}
\end{figure}

The phenomenon of \textit{fooling examples}, where deep neural networks (DNNs) confidently misclassify inputs that appear nonsensical or unrecognizable to humans, warrants deeper investigation to better understand shortcut learning and generalization in DNNs. These examples, often perceived as random noise, are nonetheless assigned high-confidence predictions by DNNs, sometimes with accuracy exceeding 99\% \cite{nguyen2015deep, s23146378, KUMANO2023259}.

Unlike adversarial examples, as shown in Fig. \ref{fig:kayleedemo} panel (a), which result from small, deliberate perturbations to the input \cite{goodfellow2014explaining}, fooling examples involve misclassifications where DNNs assign meaningful labels despite the lack of any discernible structure in the inputs. An example of this misclassification phenomenon is illustrated in Fig. \ref{fig:kayleedemo} panel (b), underscoring a paradox within traditional supervised learning, where meaningful labels are typically viewed as essential for effective model training.

Existing approaches, such as adversarial training, address adversarial examples by augmenting the training dataset with perturbed samples, thereby improving model robustness \cite{zhang2019adversarial}. However, fooling examples pose a more fundamental challenge: they inherently lack coherent labels and cannot be incorporated into conventional supervised learning frameworks. This raises a critical question: \textit{How do DNNs achieve high-confidence generalization on inputs that appear entirely random or meaningless?}

This issue underscores the need for a deeper understanding of the mechanisms that enable DNNs to make confident predictions in the absence of meaningful labels, particularly within unsupervised learning contexts. Gaining insights into these mechanisms is crucial not only for explaining how DNNs generalize effectively from seemingly random inputs but also for uncovering the fundamental principles that govern their learning dynamics and behavior.

\subsection{Summary of Contributions}
Our main contributions in this work can be summarized as follows:

\begin{enumerate}
    \item We derive an analytical solution for maximum likelihood estimation (MLE), demonstrating its existence within a feedforward linear neural network architecture. This can be achieved by maximizing the norm of the input samples in an unsupervised setting, without necessarily relying on labeled data and numerical optimization methods such as gradient-based techniques (e.g., SGD) (see \textbf{Propositions} \ref{propo:1}, \ref{propo:2}, and \ref{eq:proposition13}).
    
    \item We perform a convergence analysis, revealing the collapse of the output feature space and an improvement in model generalization, which enhances clustering of input data under an overparameterized regime (see \textbf{Observation} \ref{claim:1}). This phenomenon is explained by the convergence of angles between pairwise features, which approach zero in an extended vector space due to normalization following norm maximization. This process incorporates layer information and transitions the derived analytical MLE into an approximation in the extended vector space (see \textbf{Theorem} \ref{th:1}).
    
    \item While feature collapse improves clustering performance, it ultimately leads to a state of degeneracy where even distant inputs are mapped to the same output, resulting in zero loss. We establish the existence of a \textit{wormhole} solution (see \textbf{Theorem} \ref{thr:wormhole}) that bypasses this degeneracy and enables the reconciliation of an arbitrary random label \( P' \) with a meaningful (human-recognizable) label \( P \). This insight provides a novel perspective on how DNNs generalize effectively from seemingly random or unstructured inputs. A formal definition of label reconciliation is presented in \textbf{Definition} \ref{def:12fg}.
\end{enumerate}

\section{Deep Neural Network Optimization via Maximum Likelihood Estimation}
Before we delve into the details of our methodology, let’s briefly outline the fundamental structure of a deep neural network.

\subsection{Definition: Feedforward Linear Neural Network}
A typical deep neural network architecture consists of at least \( L \geq 1 \) layers, following a feedforward design. This architecture can be described as: 
\begin{align}\label{eq:nonlinearcnn}
H(x) = g_L \circ g_{L-1} \circ \ldots \circ g_1(x),
\end{align}
where each layer function \( g_{\ell} : \mathbb{R}^{d_1} \to \mathbb{R}^{d_{\ell+1}} \) for \( \ell = 1, 2, \ldots, L \) is typically nonlinear, expressed as
\begin{align}\label{eq:nonlinearcnn2}
g_{\ell}(x) = \sigma(G_{\ell} x + b_{\ell}),
\end{align}
with weight matrices \( G_{\ell} \in \mathbb{R}^{d_{\ell+1} \times d_{\ell}} \), biases \( b_{\ell} \in \mathbb{R}^{d_{\ell+1}} \), and the non-linear activation function (acting element-wise) \( \sigma : \mathbb{R} \to \mathbb{R} \).

In the context of supervised learning, given training data \( (x^{(i)}, y^{(i)})_{i=1}^{N} \), where inputs \( x^{(i)} \in \mathbb{R}^{d_1} \) and labels \( y^{(i)} \in \mathbb{R}^{d_L} \), the optimization objective is to minimize the following expression:
\begin{align}\label{eq:cnnopt}
\min_{G_L, \ldots, G_1} \frac{1}{N} \sum_{i=1}^{N} \mathcal{L}(H(x^{(i)}), y^{(i)}),
\end{align}
where \( \mathcal{L}: \mathbb{R}^{d_L \times d_L} \to \mathbb{R}_{+} \) is a loss function, such as the L2 norm.

This optimization is typically performed using stochastic gradient descent (SGD) and backpropagation to iteratively update the weights \( G_L, G_{L-1}, \ldots, G_1 \) to minimize the overall loss. 

Because of the nonlinear nature of Eq. \eqref{eq:nonlinearcnn} and Eq. \eqref{eq:nonlinearcnn2}, the optimization problem becomes complex. Recent theoretical pursuits have shifted attention toward a simplified scenario involving linear neural networks \cite{arora2019implicit, chou2024gradient}, where \( \sigma(x) = x \) and \( b_{\ell} = 0 \). This simplification transforms Eq. \eqref{eq:nonlinearcnn} to
\begin{align}\label{eq:linearlosscnn}
H_{\text{Linear}}(x) = G_L G_{L-1} \ldots G_1 x.
\end{align}
Considering \( x^{(i)} \) and \( y^{(i)} \) as the \( i \)-th input-label pair respectively, the optimization objective in Eq. \eqref{eq:cnnopt} can be interpreted as minimizing the pairwise distance between these vectors. An example of such a distance measure could be the angle differences, which can be described as
\begin{align}\label{eq:linearloss6}
\min_{G_L, \ldots, G_1} \frac{1}{N} \sum_{i=1}^{N} \frac{1}{\pi} \arccos \left( \frac{H_{\text{Linear}}(x^{(i)})}{\|H_{\text{Linear}}(x^{(i)})\|} \cdot \frac{y^{(i)}}{\|y^{(i)}\|} \right).
\end{align}

\subsection{An Analytical Solution for Maximum Likelihood Estimation}

The goal of maximum likelihood estimation (MLE) in DNN training is to maximize the likelihood that the inferred parameter, \( \theta \), matches the true parameter, \( \theta_0 \).

Given a training sample \( x \in \mathcal{D} \), the task is to find a function \( f \) that maps the sample statistics to the inferred parameter \( \theta \), such that \( \theta = f(x_1, \ldots, x_N) \) for all \( x \in \mathcal{D} \). If an analytical solution is not possible, numerical methods like gradient-based techniques (e.g., SGD) or gradient-free methods (e.g., genetic algorithms, Bayesian optimization) can be used.

Nevertheless, we here present an analytical solution for MLE in DNN training, without necessarily relying on numerical optimization methods. Specifically, maximizing the likelihood can be represented as follows:
\begin{align}
\theta := \arg\max_{\theta} \mathcal{L}(\theta; \mathcal{D}),
\end{align}
where \( \theta \) is the inferred parameter, and the likelihood \( \mathcal{L}(\theta; \mathcal{D}) \), typically referred to as the joint density of \( x \in \mathcal{D} \), can be described as a function of \( \theta \):
\begin{align}\label{eq:logbothsite}
&\mathcal{L}(\theta; \mathcal{D} = x_1, \ldots, x_N) = f(\mathcal{D} = x_1, \ldots, \mathcal{D} = x_N; \theta_1, \ldots, \theta_N) \nonumber\\
&= \prod_{i=1}^{N} f(\mathcal{D} = x_i; \theta).
\end{align}
The second line of Eq. \eqref{eq:logbothsite} follows the standard assumption in machine learning that all training samples \( (x_1, \ldots, x_N) \in \mathcal{D} \) are independent and identically distributed (i.i.d.). Therefore, each \( x \in \mathcal{D} \) is governed by the same parameter \( \theta \). Maximizing the likelihood is typically achieved by \textit{minimizing the negative log-likelihood}, expressed as: 
\begin{align}\label{eq:solveoptimize}
- \frac{1}{N} \log \mathcal{L}(\theta; \mathcal{D}) = - \frac{1}{N} \sum_{i=1}^{N} \log(f(\mathcal{D} = x_i; \theta)),
\end{align}
where the likelihood is averaged by dividing by \( N \).

However, directly solving this minimization problem does not provide meaningful insights into the true underlying distribution \( \mathcal{D} \), which should ideally correspond to the true parameter \( \theta_0 \). We give below a proposition to argue the existence of an analytic solution for Eq. \eqref{eq:solveoptimize}.

\begin{Proposition}\label{propo:1}
For a set of training samples \( (x_1, \ldots, x_N) \in \mathcal{D} \), where each \( x_i \) is a vector of length \( k < n \) drawn from a distribution \( \mathcal{D} \in \mathbb{R}^k \), there exists an analytical maximum likelihood solution such that:
\begin{align}\label{eq:resugejio39}
-\frac{1}{n} \log_2 \mathcal{L}(\theta; \mathcal{D}) = H_2(\theta_0) \in (0, 1),
\end{align}
where the sample statistics \( f(x \in \mathcal{D}; \theta_0) \) follow a Binomial distribution if and only if the inferred parameter equals the true parameter, i.e., \( \theta = \theta_0 = \frac{k}{n} \), where \( H_2(\theta_0) = -\theta_0 \log_2(\theta_0) - (1 - \theta_0) \log_2(1 - \theta_0) \) is the binary cross-entropy.
\end{Proposition}
\begin{proof}
We begin by manipulating the log-likelihood expression. We add and subtract the log-likelihood terms $\log(f(\mathcal{D} = x_i; \theta_0))$. This adjustment is valid because it adds zero to the equation, leaving the equation's intrinsic meaning unchanged. Thus, we can rewrite Eq. \eqref{eq:solveoptimize} in terms of the true parameter statistic as:
\begin{align}\label{eq:convergekl}
&\frac{1}{N} \left[ \sum_{i=1}^{N} \log \left( \frac{f(\mathcal{D} = x_i; \theta_0)}{f(\mathcal{D} = x_i; \theta)} \right) - \sum_{i=1}^{N} \log(f(\mathcal{D} = x_i; \theta_0)) \right].
\end{align}
Now, let us assume that all training samples $(x_1, \ldots, x_N) \in \mathcal{D}$ are i.i.d.. For any $x \in \mathcal{D}$, we use the following description for the parameter-inferred function:
\begin{align}\label{eq:densityinfer}
f(x; \theta) = \prob{x = k} = \theta^k(1-\theta)^{n-k}.
\end{align}
The true function corresponding to $\theta_0$ is given by:
\begin{align}\label{eq:densitytrue}
f(x; \theta_0) = {n \choose k} \theta_0^k (1 - \theta_0)^{n - k}.
\end{align}
This formulation helps establish a relationship between the inferred and true function. Specifically, we have:
\begin{align}\label{eq:stringentrelation}
\frac{f(x ; \theta_0)}{f(x ; \theta)} = {n \choose k}, \quad \text{if and only if} \quad \theta_0 = \theta.
\end{align}
Now, applying the logarithm $\log_2(.)$ to Eq. \eqref{eq:stringentrelation}, we substitute it into Eq. \eqref{eq:convergekl}, which gives:
\begin{align}\label{eq:intermediatelogi}
- \log_2 \mathcal{L}(\theta; \mathcal{D}) = -\log_2 \left( \theta_0^k (1-\theta_0)^{n-k} \right) = n H_2(\theta_0),
\end{align}
where $H_2(\theta_0)$ is the binary cross-entropy function.

Dividing both sides of Eq. \eqref{eq:intermediatelogi} by $n$, we normalize the result to the interval $(0,1)$, leading to Eq. \eqref{eq:resugejio39}.

Finally, to find the optimal value of $\theta_0$ that minimizes the average negative log-likelihood, we differentiate Eq. \eqref{eq:intermediatelogi} with respect to $\theta_0$ and set the derivative equal to zero:
\begin{align} \label{eq:fjfie999c}
\frac{d}{d\theta_0} \left[ -\log_2 (\theta_0^k (1 - \theta_0)^{n - k}) \right] = 0  
\end{align}
which simplifies to:
\begin{align}
\frac{\theta_0}{1 - \theta_0} = \frac{k}{n - k}, \quad \text{leading to} \quad \theta_0 = \frac{k}{n}.
\end{align}
Thus, the optimal solution for $\theta_0$ is $\frac{k}{n}$, which minimizes the average negative log-likelihood, completing the proof.
\end{proof}

It is important to note that Proposition \ref{propo:1} requires that any input $x\in{\mathcal{D}}$ to be i.i.d, where the true function for the sample statistics is described as a Binomial distribution. Moreover, the parameter-inferred function represents a specific case of the true function. Below, we present an additional proposition that shows how the analytical maximum likelihood solution can be expressed in terms of the inner product of arbitrary pairs of vectors $w \in \mathbb{R}^{k}$ and $w' \in \mathbb{R}^{k}$.

\begin{Proposition}\label{propo:2}
Let \( w \in \mathbb{R}^{k} \) and \( w' \in \mathbb{R}^{k} \) be two random vectors. Define \( v_i \in \mathbb{R}^{k} \) as the random unit vectors for \( i = 1, 2, \dots, n \), where each \( v_i \) is independently drawn from the standard normal distribution \( \mathcal{N}(0,1) \). Also, define \( \beta_i = \mathbf{1}_{\mathsf{sgn}(v_i^T \cdot w) \neq \mathsf{sgn}(v_i^T \cdot w')} \) as an indicator function, where \( \beta_i = 1 \) if \( \mathsf{sgn}(v_i^T \cdot w) \neq \mathsf{sgn}(v_i^T \cdot w') \) and \( \beta_i = 0 \) otherwise, and \( \mathsf{sgn}(\cdot) \) is the Signum function. If the Hamming distance, denoted as \( x \), between the transformed sample pair \( v_i^T w \) and \( v_i^T w' \) for \( i = 1, 2, \ldots, n \), satisfies:
\begin{align}
x = \sum_{i=1}^{n} \beta_i = \sum_{i=1}^{n} \mathbf{1}_{\mathsf{sgn}(v_i^T \cdot w) \neq \mathsf{sgn}(v_i^T \cdot w')} = k,
\end{align}
then the following relationship holds:
\begin{align}\label{eq:soltuoab9}
\boxed{-\frac{1}{n} \log_2 \mathcal{L}(\theta; \mathcal{D}) = H_2\left( \frac{1}{\pi} \arccos\left( \frac{w}{\|w\|} \cdot \frac{w'}{\|w'\|} \right) \right).}
\end{align}
\end{Proposition}

\begin{proof}
We use the Cosine Distance-based Locality-Sensitive Hashing (LSH) \cite{charikar2002similarity}, where individual functions \( h_i(\cdot) \) are defined by randomly chosen unit vectors \( v_i \in \mathbb{R}^{k} \), and the Signum function is used for quantization. Specifically:
\begin{align}
& h_i(w) = \mathsf{sgn}(v_i^T \cdot w) \in \{-1, 1\}, \nonumber \\
& h_i(w') = \mathsf{sgn}(v_i^T \cdot w') \in \{-1, 1\}, \quad \text{for } i = 1, \ldots, n.
\end{align}
Each \( v_i \) is drawn randomly from the standard normal distribution \( \mathcal{N}(0, 1) \), and each function \( h_i(\cdot) \) produces a probability of difference between \( (w, w') \) in the transformed domain. This probability can be expressed as:
\begin{align}\label{eq:solgthest24}
\theta_0 = \frac{\sum_{i=1}^{n} h_i(w) \neq h_i(w')}{n} = \frac{1}{\pi} \arccos\left( \frac{w}{\|w\|} \cdot \frac{w'}{\|w'\|} \right).
\end{align}
We are interested in the case when \( \theta=\theta_0 = \frac{k}{n} \), as stated in Proposition \ref{propo:1}. This corresponds to the scenario where the Hamming distance between the two vectors (after the transformation) equals \( k \), as described by:
\begin{align}\label{eq:doo3t}
x = \sum_{i=1}^{n} \beta_i = \sum_{i=1}^{n} \mathbf{1}_{\mathsf{sgn}(v_i^T \cdot w) \neq \mathsf{sgn}(v_i^T \cdot w')} = k.
\end{align}
Clearly, \( \beta_i \) are i.i.d. random variables, and thus \( x \sim \text{Bin}(n, \theta_0) \), following a Binomial distribution:
\begin{align}\label{eq:lshbinomialdas0}
\mathbb{P}(x = k) = {n \choose k} (\theta_0)^k (1 - \theta_0)^{n-k}.
\end{align}
This formulation aligns with the description of the true statistics of \( x \) and \( \theta_0 \) given in Proposition \ref{propo:1}, where \( x \) is now described as the pairwise Hamming distance between the transformed sample pair \( v_i^T \cdot w \) and \( v_i^T \cdot w' \) (for \( i = 1, 2, \ldots, n \)).
\end{proof}

\subsection{MaxLikelihood Algorithm}
Our devised MaxLikelihood algorithm is outlined in Algorithm \ref{algo:maxlik}. This algorithm utilizes a cosine distance-based LSH within a feedforward linear neural network architecture, characterized by a maximum layer depth \( L \) and width \( n \). Briefly, the algorithm begins with the random initialization of the weight matrices \( G^*_{\ell}\in\RR^{n\times{k}} \). The optimization process involves selecting $k$ rows from \( G^*_{\ell} \) to form \( G_{\ell} \subset G^*_{\ell} \) that maximize the norm of the output vector \( w_{\ell} = G_{\ell} w_{\ell-1} \in{\RR^{k}}\) for each layer \( \ell = 1, 2, \dots, L \). Consequently, as the number of layers \( L \) increases, the norm \( \| w_L \| \) will increase proportionally\footnote{As \( L \) increases and becomes sufficiently large, the norm \( \| w_L \| \) will approach infinity. However, we can perform normalization by dividing \( w_\ell \) by its norm, \( \frac{w_\ell}{\| w_\ell \|} \), at the output of each layer to prevent each intermediate output vector from becoming infinite.
}. 

We present the following proposition to formalize the optimization process described above, emphasizing that an analytical solution for MLE exists within a feedforward neural network architecture of width \( n \) and depth \( L \):

\begin{Proposition}\label{eq:proposition13}
Let \(w \in \mathbb{R}^k \) denote an input vector for Algorithm \ref{algo:maxlik}, and let \( G^*_{\ell} \in \mathbb{R}^{n \times k} \) represent the weight matrix of the \(\ell\)-th layer of a feedforward linear neural network with width \(n\), where the entries of \(G^*_{\ell}\) follow a standard normal distribution \( \mathcal{N}(0,1) \). Let \( G_{\ell} \subset G^*_{\ell} \in \mathbb{R}^{k \times k} \) denote a submatrix of \( G^*_{\ell} \), obtained by selecting \( k \) rows from the original weight matrix. For an arbitrary test vector \(w'\in\mathbb{R}^k\), by maximizing the norm \(
\norm{w_{\ell}}=\| G_{\ell} w_{\ell-1} \|\) for each layer \( \ell = 1, \dots, L \), we derive:
\begin{align}\label{eq:solution1}
-\frac{1}{n}\log_2{\mathcal{L}(\theta;\mathcal{D})} &= H_2\left(\frac{1}{\pi} \arccos\left( \frac{w_{L}}{\|w_{L}\|} \cdot \frac{w'_{L}}{\|w'_{L}\|} \right) \right) \nonumber \\
&= H_2\left( \frac{1}{\pi} \arccos\left( \frac{w_{L-1}}{\|w_{L-1}\|} \cdot \frac{w'_{L-1}}{\|w'_{L-1}\|} \right) \right) \nonumber \\
&= \ldots \nonumber \\
&= H_2\left( \frac{1}{\pi} \arccos\left( \frac{w_{1}}{\|w_{1}\|} \cdot \frac{w'_{1}}{\|w'_{1}\|} \right) \right) \nonumber \\
&= H_2\left( \frac{1}{\pi} \arccos\left( \frac{w}{\|w\|} \cdot \frac{w'}{\|w'\|} \right) \right).
\end{align}
\end{Proposition}

\begin{proof}
As described in Eq. \eqref{eq:stringentrelation}, the parameter-inferred function can be considered a special case of Eq. \eqref{eq:lshbinomialdas0}. In this context, we examine the worst-case scenario. Among the \( \binom{n}{k} \) possible configurations where \( x = \sum_{i=1}^{n} \mathbf{1}_{\mathsf{sgn}(v_i^T \cdot w) \neq \mathsf{sgn}(v_i^T \cdot w')} = k \), we select the configuration that maximizes \( \sum_{i=1}^{k} \|v_i^T \cdot w\| \). The solution is as follows:
\begin{align}\label{eq:fei002r}
x &= \sum_{i=1}^{k} \mathbf{1}_{\max \|v_i^T \cdot w\| \neq v_i^T \cdot w'} = k, \nonumber \\
\text{for} \quad v_i^T \cdot w' &= -\max \|v_i^T \cdot w\|, \quad i = 1, 2, \ldots, k.
\end{align}
We omit the Signum function \( \mathsf{sgn}(\cdot) \) in Eq. \eqref{eq:fei002r} because, under the condition \( v_i^T \cdot w' = -\max \|v_i^T \cdot w\| \), it is clear that \( \mathsf{sgn}(\max \|v_i^T \cdot w\|) \neq \mathsf{sgn}(v_i^T \cdot w') \). Consequently, we have \( x = \sum_{i=1}^{k} \mathbf{1}_{\mathsf{sgn}(\max \|v_i^T \cdot w\|) \neq \mathsf{sgn}(v_i^T \cdot w')} = k \).

Next, we define a maximum likelihood optimization by iteratively maximizing \( \|G_{\ell} w_{\ell-1}\| \) for each weight matrix \( G_{\ell} \in \mathbb{R}^{k \times k} \), input vector \( w_{\ell-1} \in \mathbb{R}^{k} \), and \( \ell = 1, 2, \ldots, L \). The entries of \( G_{\ell} \) follow a standard normal distribution with zero mean and unit variance \( \mathcal{N}(0,1) \). We can represent the \( i \)-th row of \( G_{\ell} \) as \( v_i^T \). Maximizing the norm \( \|G_{\ell} w_{\ell-1}\| \) is achieved by initializing \( G^*_{\ell} \sim{\mathcal{N}(0,1)}\) and selecting \( k \) rows to form \( G_{\ell} \) such that \( \|G_{\ell} w_{\ell-1}\| \) is maximum, as shown in Algorithm \ref{algo:maxlik}.

By normalizing the output vectors to unit norm, we express Eq. \eqref{eq:fei002r} in terms of the worst-case Hamming distance \( d_H(\cdot, \cdot) \) between the normalized output vectors \( \frac{w_{\ell}}{\|w_{\ell}\|} \) and \( \frac{w'_{\ell}}{\|w'_{\ell}\|} \):
\begin{align}\label{eq:fei069702r}
&x = \sum_{i=1}^{k} \mathbf{1}_{\max \|v_i^T \cdot w\| \neq v_i^T \cdot w'} = d_H\left(\frac{\max(w_{\ell})}{\|\max(w_{\ell})\|}, \frac{w'_{\ell}}{\|w'_{\ell}\|}\right) = k, \nonumber \\
&\text{for} \quad \frac{\max(w_{\ell})}{\|\max(w_{\ell})\|} = -\frac{w'_{\ell}}{\|w'_{\ell}\|}, \quad \ell \geq 1.
\end{align}
For the case where \( \ell = L = 1 \), the above condition implies that \( \frac{\max(w_{\ell})}{\|\max(w_{\ell})\|} = \frac{w_{L}}{\|w_{L}\|} =-\frac{w'_{L}}{\|w'_{L}\|} \), where \( x = d_H\left(\frac{w_{L}}{\|w_{L}\|}, \frac{w'_{L}}{\|w'_{L}\|}\right) = k \), yielding a corresponding solution described in Eq. \eqref{eq:soltuoab9} for \( H_2\left( \frac{1}{\pi} \arccos\left( \frac{w_{L-1}}{\|w_{L-1}\|} \cdot \frac{w'_{L-1}}{\|w'_{L-1}\|} \right) \right)=H_2\left( \frac{1}{\pi} \arccos\left( \frac{w}{\|w\|} \cdot \frac{w'}{\|w'\|} \right) \right) \). The same reasoning extends to all \( \ell \geq 1 \), leading to 
\begin{align}
x &= d_H\left(\frac{w_{L}}{\|w_{L}\|}, \frac{w'_{L}}{\|w'_{L}\|}\right) = d_H\left(\frac{w_{L-1}}{\|w_{L-1}\|}, \frac{w'_{L-1}}{\|w'_{L-1}\|}\right) \nonumber \\
&=\ldots = d_H\left(\frac{w_{2}}{\|w_{2}\|}, \frac{w'_{2}}{\|w'_{2}\|}\right) = d_H\left(\frac{w_{1}}{\|w_{1}\|}, \frac{w'_{1}}{\|w'_{1}\|}\right) = k,
\end{align}
thus yielding Eq. \eqref{eq:solution1} for \( \ell = 1, 2, \ldots, L \).
\end{proof}

\begin{algorithm}
\caption{MaxLikelihood Optimization Algorithm}\label{algo:maxlik}
\begin{algorithmic}[1]
\Function{MaxLikelihood}{$w, n, L$}
    \State Initialize $\ell = 1$, 
    \State Set $w_{\ell -1}=w$
    \State Set $k$ as the length of $w_{\ell -1}$
    \While{$\ell  \leq L$} 
        \State Initialize a random matrix $G^*_{\ell } \in \mathbb{R}^{n \times k}$ \Comment{$G^*_{\ell} \sim \mathcal{N}(0,1)$}
        \State Compute $x_{\ell}=\argmax_{{w_{\ell}\subset{G^*_{\ell}w_{\ell-1}}} }{\norm{G^*_{\ell}w_{\ell-1}}}$
        \State Record $G_{\ell} \subset G^*_{\ell}$  \Comment{Dropout step occur here s.t. $G_{\ell}w_{\ell-1} = w_{\ell}$, where $G_{\ell}\in{\RR^{k\times{k}}}$}
        \State Set $\ell = \ell + 1$
    \EndWhile
    \State \textbf{Return} output model $H_{\textnormal{Linear}}(\cdot) = G_L\ldots G_2 G_1(\cdot)$, and output vector $w_{L}\in{\RR^{k}}$ \Comment {$w_{L} = H_{\textnormal{Linear}}(w)$}
\EndFunction
\end{algorithmic}
\end{algorithm}

 \section{Convergence Testing}\label{sec:testconvergedidsafggagram}
\begin{figure}[!ht]
\centering
  \includegraphics[scale=0.46]{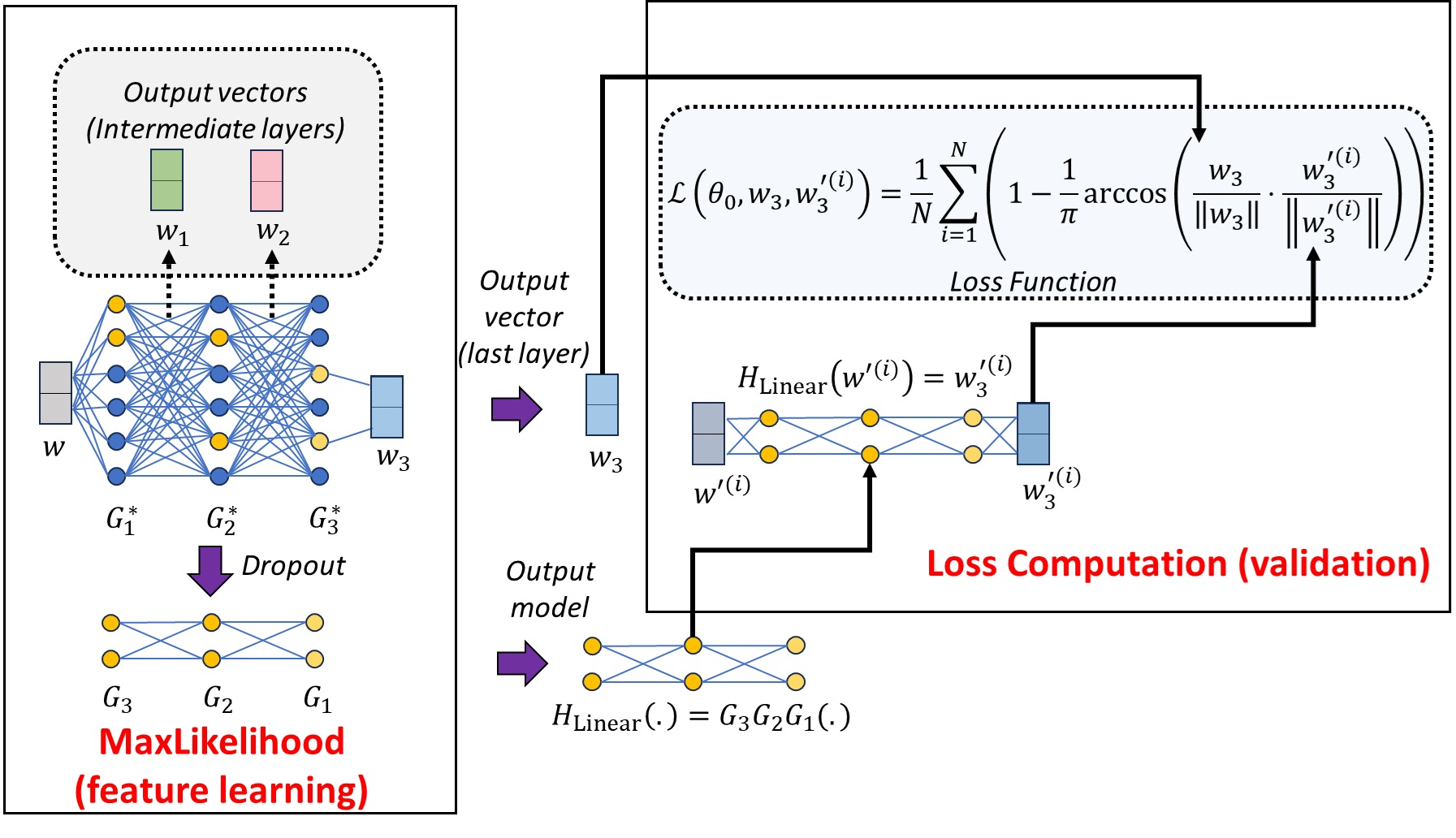}
  \caption{Convergence test with MaxLikelihood algorithm (with input length \( k = 2 \) and number of layer $L=3$ used here as an example for illustration purpose).}\label{fig:testconvergediagram}
\end{figure}
In this section, we present a test to evaluate the convergence properties of Algorithm \ref{algo:maxlik}.
\subsection{Convergence Test Set-up}
In our convergence test, as illustrated in Fig. \ref{fig:testconvergediagram}, an input sample \( w \in \mathbb{R}^k \) of arbitrary length \( k \) is provided to the MaxLikelihood algorithm (Algorithm \ref{algo:maxlik}) to extract its corresponding output vector \( w_L \) at the final layer \( \ell = L \). Subsequently, dropout is applied to yield the output model \( H_{\text{Linear}}(.) = G_L \ldots G_2 G_1(.) \), which is then used for testing on new, unseen samples \( w'^{(i)} \) (for \( i = 1, 2, \ldots, N \)).

Our convergence test aims to demonstrate the capability of the proposed optimization in satisfying the condition outlined in Eq. \eqref{eq:fei069702r}. More precisely, by focusing solely on the output vector of the last layer, we require:
\begin{align}
\frac{w_{L}}{\norm{ w_{L}}} = -\frac{w'_{L}}{\norm{w'_{L}}}, \quad \ell = L,
\end{align}
which implies a state of \textit{perfect anti-correlation} between the output vector \( w_L \) of maximum norm and an arbitrary test vector \( w'_L \), after both are normalized to unit norm. With this requirement, it is straightforward to validate the convergence of the proposed optimization algorithm by choosing an objective (loss) function described as:
\begin{align}\label{eq:ffshau93}
\mathcal{L}(\theta^{(i)}_0,w_L, w'^{(i)}_L)=\frac{1}{N}\sum_{i}^{N}1-\frac{1}{\pi}\arccos(\frac{w_L}{\norm{w_L}}\frac{w'^{(i)}_L}{\norm{w'^{(i)}_L}})\in{(0,1)}.
\end{align}
The convergence of the objective function toward zero implies \( \frac{w_{L}}{\norm{w_{L}}} = -\frac{w'^{(i)}_{L}}{\norm{w'^{(i)}_{L}}} \) for \( i = 1, \ldots, N \) test samples, as \( \arccos(-1) = \pi \).

 \begin{figure*}[!htp]
\centering
  \includegraphics[scale=0.90]{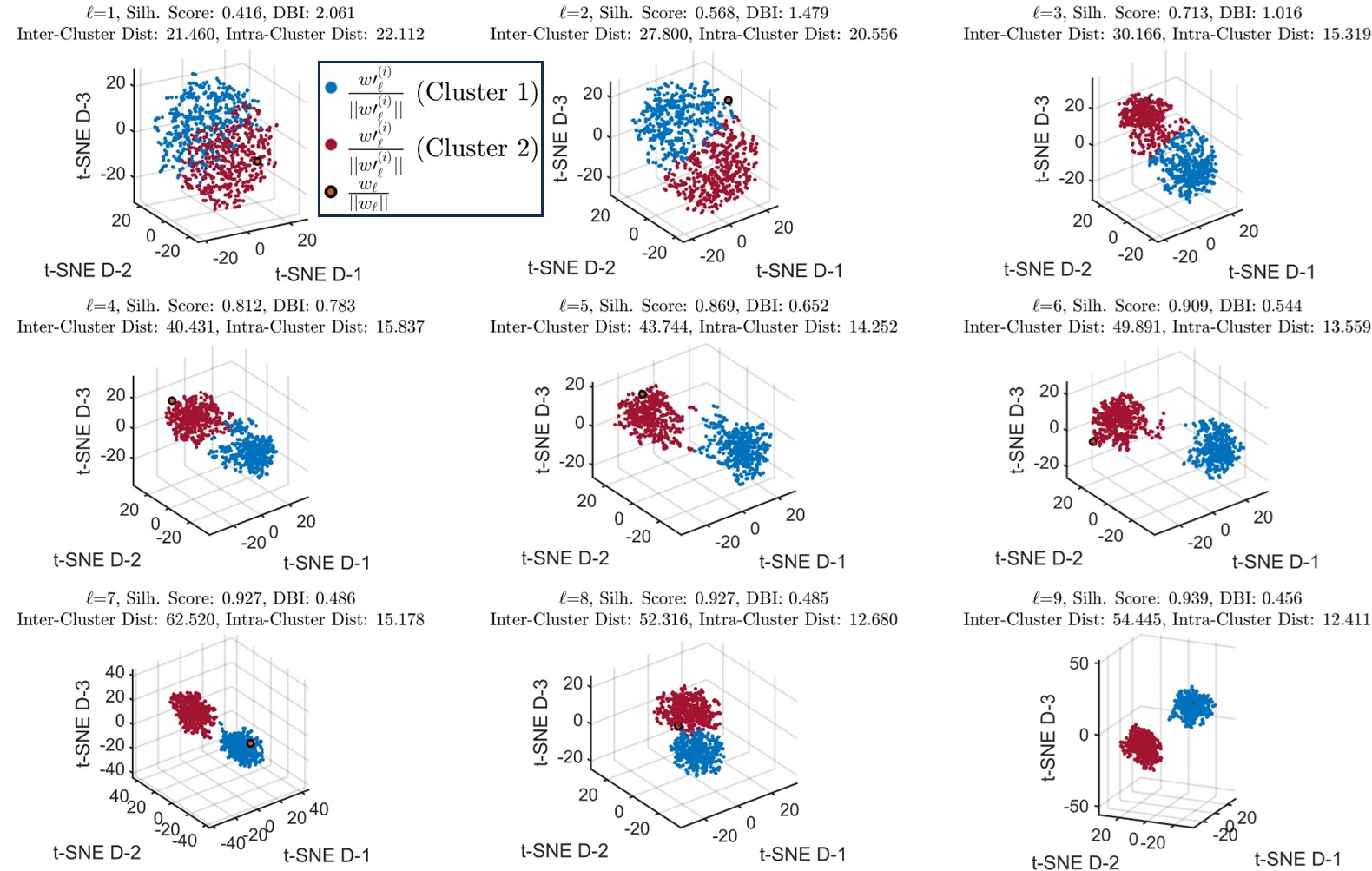}
  \caption{t-SNE low-dimension (3-D) representation of the normalized features, with $L=9$.}
  \label{fig:6716}
\end{figure*}
\begin{figure*}[!htp]
\centering
  \includegraphics[scale=0.70]{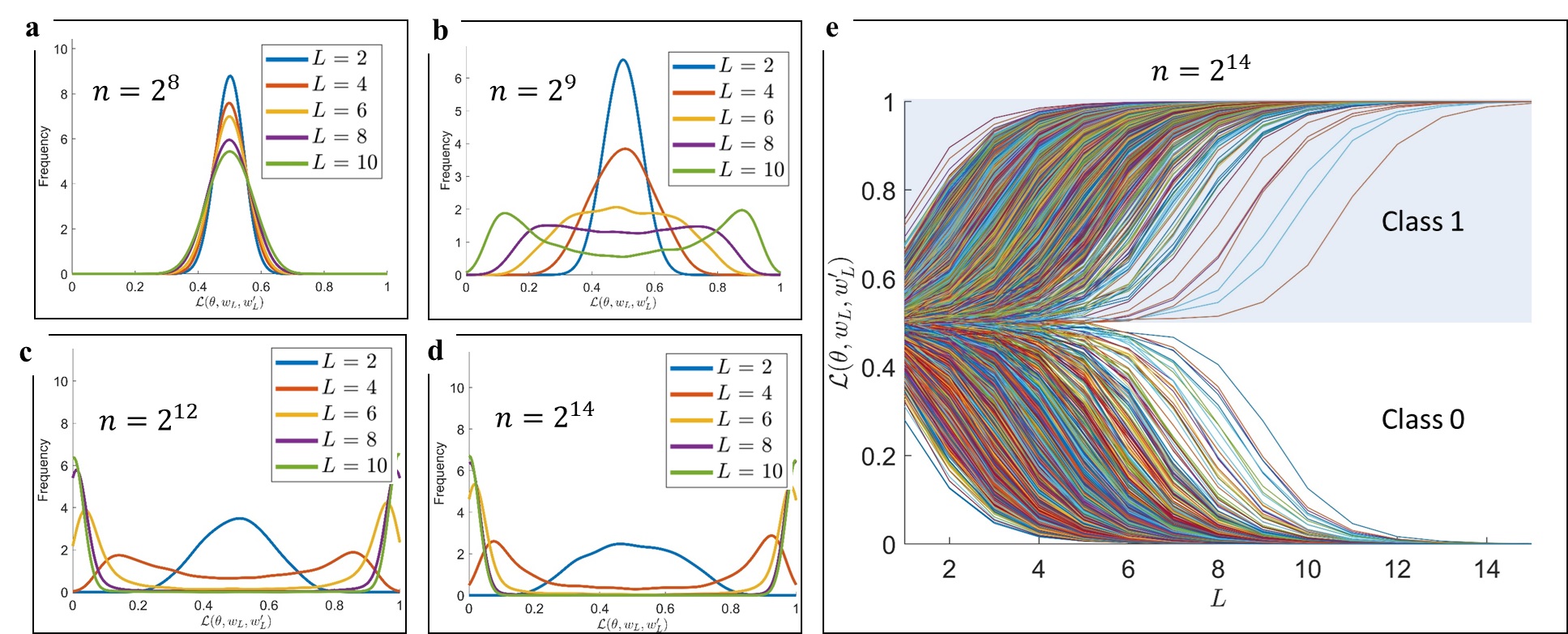}
  \caption{Results of the convergence test, showing the distribution of loss values and the convergence patterns.}
\label{fig:testconvergencol2}
\end{figure*}
\begin{figure}[!htp]
\centering
  \includegraphics[scale=0.46 ]{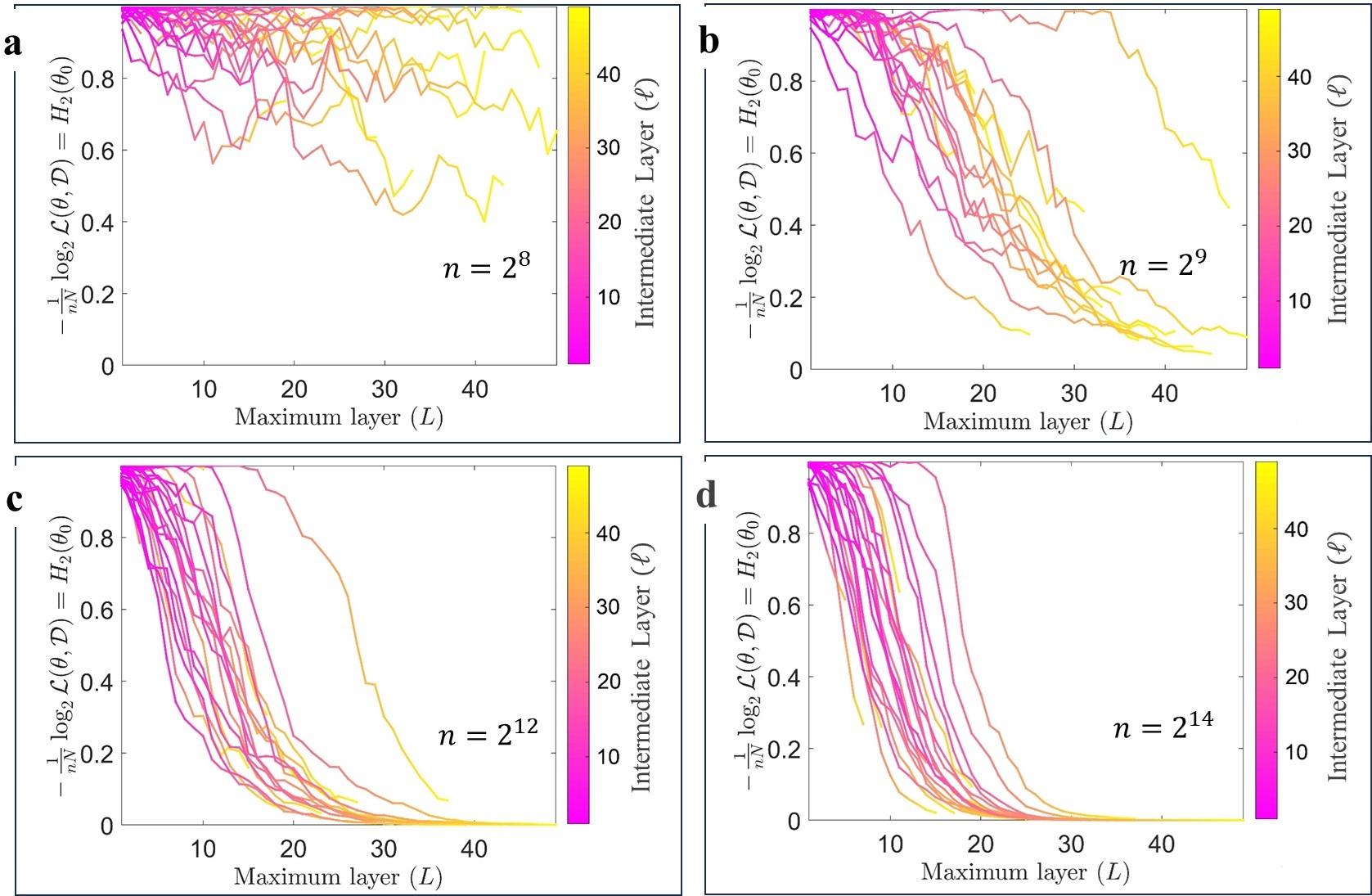}
  \caption{Convergence of the negative log-likelihood toward zero. Testing is repeated for increasing $n$ and \(L=1,3,5,7,\ldots,49\).}
  \label{fig:testconvergediagramglobal}
\end{figure}
  
\subsection{t-SNE Plot and Convergence Test Results}

\textbf{t-Distributed Stochastic Neighbor Embedding (t-SNE) Plot:} Figure \ref{fig:6716} displays a t-SNE plot of the normalized output vectors \( \frac{w_L}{\|w_L\|} \) (from the feature learning phase) and \( \frac{w'^{(i)}_L}{\|w'^{(i)}_L\|} \) (from the validation phase), where \( i = 1, 2, \dots, N \) and \( k = 256 \). For this visualization, we randomly select a subset of \( N = 5000 \) test samples drawn from a Gaussian distribution with zero mean and unit variance. The t-SNE method projects these high-dimensional vectors into a lower-dimensional space, allowing us to visualize their distribution. The k-means algorithm is applied to identify potential clusters within the test vectors \( \frac{w'^{(i)}_{\ell}}{\|w'^{(i)}_{\ell}\|} \) as they progress through deeper layers of the model (i.e., as \( \ell \rightarrow L \)).

At the first layer (\( \ell = 1 \)), the output vectors are evenly distributed within a confined region. As \( \ell \) increases, the vectors undergo successive optimization, and by the final layer (\( \ell = L \)), they collapse into two well-separated clusters—one centered around \( \frac{w_L}{\|w_L\|} \) and the other displaced further away.

Clustering performance is evaluated using four key metrics: Silhouette Score, Davies-Bouldin Index (DBI), Inter-Cluster Distance, and Intra-Cluster Distance. The Silhouette Score assesses cohesion and separation, with higher values indicating better clustering. The DBI measures the ratio of intra-cluster to inter-cluster distances, with lower values suggesting more distinct clusters. Inter-Cluster Distance quantifies the separation between clusters, while Intra-Cluster Distance evaluates cluster compactness.

As \( \ell \to L \), all metrics show consistent improvement. Silhouette Scores increase, indicating better cohesion and separation. The DBI decreases, signifying more distinct clusters. Inter-Cluster Distances increase, demonstrating improved separation, while Intra-Cluster Distances decrease, indicating tighter clusters. These trends collectively demonstrate a clear enhancement in clustering performance, suggesting that higher values of \( \ell \) lead to more accurate and well-defined clusters.

\textbf{Convergence Test Result:} The convergence test results are shown in Fig. \ref{fig:testconvergencol2}. For large values of \( L \) and \( n \) (see panels (a), (b), (c), and (d)), we observe the same collapsing phenomenon, where the computed loss \( \mathcal{L}(\theta_0, w_L, w'_L) \) divides into two distinct groups, converging to either zero or one. This collapse follows a clear, symmetrical pattern. As shown in panel (e), loss values below 0.5 (the ideal boundary between the two classes) converge toward zero, while those above 0.5 converge toward one. This behavior simplifies the classification task, as the classifier can easily select the class whose mean is closest to the test vector.

Based on these observations, we can make the following formal statement:

\begin{Observation} \label{claim:1}
By selecting sufficiently large values of \( n \) and \( L \)—typically representing an overparameterized regime with more parameters than training samples—the output model $H_{\text{Linear}}(.)$ can achieve arbitrarily close to zero loss, leading to a perfect anti-correlation scenario. Specifically, the loss function (described in Eq. \eqref{eq:ffshau93}) has a solution described as:

\begin{align}\label{eq:23antico4}
&\mathcal{L}(\theta^{(i)}_0, w_L, w'^{(i)}_L) \nonumber \\
&= \frac{1}{N} \sum_{i=1}^{N} \left( 1 - \frac{1}{\pi} \arccos \left( \frac{w_L}{\|w_L\|} \cdot \frac{w'^{(i)}_L}{\|w'^{(i)}_L\|} \right) \right), \nonumber \\
&= \frac{1}{N} \sum_{i=1}^{N} \left( 1 - \theta^{(i)}_0 \right)= 0, \ \textnormal{where} \ \underbrace{\frac{w_L}{\|w_L\|} = -\frac{w'^{(i)}_L}{\|w'^{(i)}_L\|}}_\textnormal{(Perfect Anti-Correlation)}.
\end{align}
Meanwhile, due to the collapse of the computed loss into two distinct groups, the following condition of perfect correlation also emerges:
\begin{align}\label{eq:23antico5}
&\mathcal{L}(\theta^{(i)}_0, w_L, w'^{(i)}_L)_+ \nonumber\\
&= \frac{1}{N} \sum_{i=1}^{N} \frac{1}{\pi} \arccos \left( \frac{w_L}{\|w_L\|} \cdot \frac{w'^{(i)}_L}{\|w'^{(i)}_L\|} \right), \nonumber \\
&= \frac{1}{N} \sum_{i=1}^{N} \theta^{(i)}_0 = 0, \ \textnormal{where} \ \underbrace{\frac{w_L}{\|w_L\|} = \frac{w'^{(i)}_L}{\|w'^{(i)}_L\|}}_\textnormal{(Perfect Correlation)}.
\end{align}
\end{Observation}
Note that zero losses in Eq. \eqref{eq:23antico4} and Eq. \eqref{eq:23antico5} correspond to zero value in the negative log-likelihood, as seen in Eq. \eqref{eq:soltuoab9}, due to the symmetry of \( H_2(\theta^{(i)}_0) = H_2(1 - \theta^{(i)}_0) \). Thus, these computed losses align with the analytical optimal solution derived from MLE.

Following the same convergence testing setup, the convergence of the negative log-likelihood toward zero as \( n \) and \( L \) increase is shown in Figure \ref{fig:testconvergediagramglobal}.

\section{Unsupervised Neural Collapse: Insight and Application} \label{sec:negeah9}
The observed collapsing behavior in our convergence test demonstrates robust generalization, as it indicates that the model has learned a structured feature space capable of distinguishing between familiar (near-zero loss) and unfamiliar (near-one loss) test samples, even those unseen during feature learning. As the model progresses with increasing width \( n \) and depth \( L \), there is an increasing marginal separation between these two classes.

Given the empirical evidence, we can relate this collapsing phenomenon to Neural Collapse (NC), as illustrated in Fig. \ref{fig:nca}, a notable phenomenon first introduced by Papyan et al. \cite{papyan2020prevalence} (for a comprehensive review, see Kothapalli et al. \cite{kothapalli2022neural}). 

\begin{figure}[!htp]
\centering
  \includegraphics[scale=0.15]{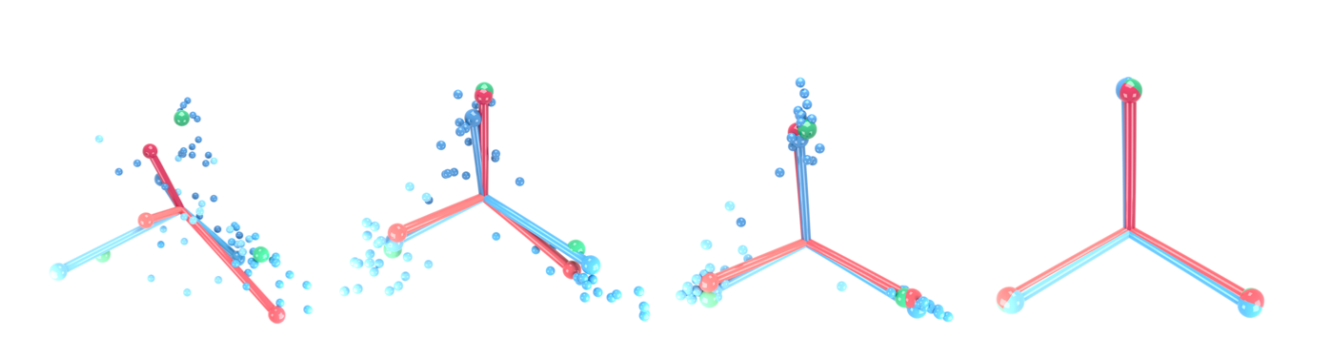}
  \caption{The neural collapse phenomenon (image sourced from and credited to \cite{papyan2020prevalence}).}
  \label{fig:nca}
\end{figure}

NC describes four key properties that emerge in the final layers (\( \ell = L \)) of deep neural networks trained beyond zero error. Our results, shown in Fig. \ref{fig:6716} and Fig. \ref{fig:testconvergencol2} panels (a), (b), (c), (d), (e), clearly exhibit these four properties of NC, as outlined below:

\begin{enumerate}
    \item \textbf{Collapse of variability:} Data samples from the same class converge towards their class mean in the final layer, reducing intra-class variability and collapsing towards a single-point representation.
    \item \textbf{Convergence to simplex equiangular tight frame (ETF):} The class-wise means of the last layer features align to form a simplex ETF, a symmetrical structure on a hypersphere, maximizing the distance between class means and ensuring linear separability.
    \item \textbf{Convergence to self-duality:} The class means and linear classifiers converge towards each other, potentially with rescaling, illustrating self-duality in the network’s feature representation.
    \item \textbf{Simplification to nearest class center (NCC):} During classification, the linear classifier selects the class whose mean is closest to the test sample.
\end{enumerate}
\subsection{Clustering and Classification Applications}
The convergence of the negative log-likelihood to zero (Eq.~\eqref{eq:23antico4} and Eq.~\eqref{eq:23antico5}) indicates that, despite achieving global minimization, the model implicitly distinguishes between two distinct groups: one characterized by perfect correlation (class mean of zero) and the other by perfect anti-correlation (class mean of one). This distinction is crucial as it prevents overfitting, even in the overparameterized regime. Instead of merely memorizing individual data points, the model organizes the data into two well-separated clusters in the feature space. This behavior, driven by the implicit bias toward Neural Collapse, arises naturally through maximum likelihood estimation in an unsupervised setting. This mechanism functions as a regularizer, enabling the model to generalize effectively without succumbing to overfitting.

\textbf{Clustering via Memorization:} Given the strong generalization and collapsing properties of the model, we explore its potential for unsupervised learning tasks, particularly in extracting meaningful representative features from unseen test samples. Here, we demonstrate how clustering can be achieved using the output model \( H_{\text{Linear}}(\cdot) \), as illustrated in Figure~\ref{fig:lvclusterfig}. The approach is conceptually simple: it leverages the model's ability to collapse test samples into two distinct clusters while memorizing their trajectories across layers. At higher layer \( \ell\), the inter-cluster distance increases, while the intra-cluster distance decreases. By tracing and memorizing the hierarchical evolution of these clusters across layers (e.g. $\ell=1,2,\ldots,{L}$), the final output distribution comprises \( 2L \) distinct clusters. This method provides an effective means of clustering without the need for labeled data, relying on the model's ability to separate features and encode the trajectory of cluster formation, layer by layer.

Figure \ref{fig:clus6f716} illustrates the clustering results for \( L \) sets of \( N = 500 \) random test samples, \( w'^{(i)} \), where \( i = 1, 2, \dots, N \) (with each set corresponding to a different layer). These samples are randomly drawn from a Gaussian distribution with zero mean and unit variance. Each output vector is normalized to unit norm before being projected into a lower-dimensional space using t-SNE. The plot visually demonstrates the model's ability to effectively cluster the data in an unsupervised setting. To further enhance the visualization, we applied the k-means clustering algorithm, with each distinct cluster assigned a unique color.

\begin{figure}[!htp]
\centering
  \includegraphics[scale=0.225]{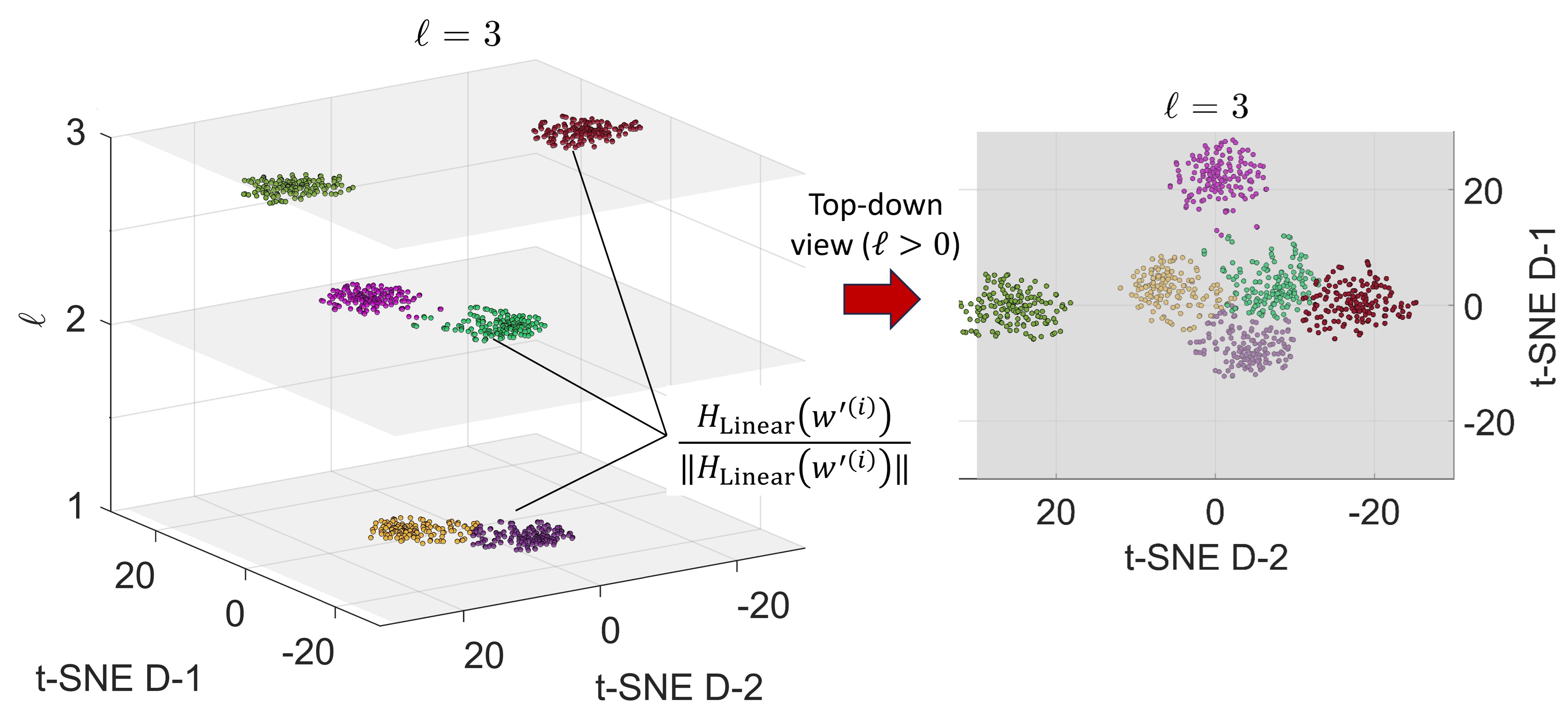}
  \caption{An analogy illustrating the output model for clustering across \( L = 3 \) layers. As $\ell\rightarrow{L}$ increases, new, tighter clusters form with greater inter-cluster separation. }
  \label{fig:lvclusterfig}
\end{figure}
\begin{figure*}[!htp]
\centering
  \includegraphics[scale=0.18]{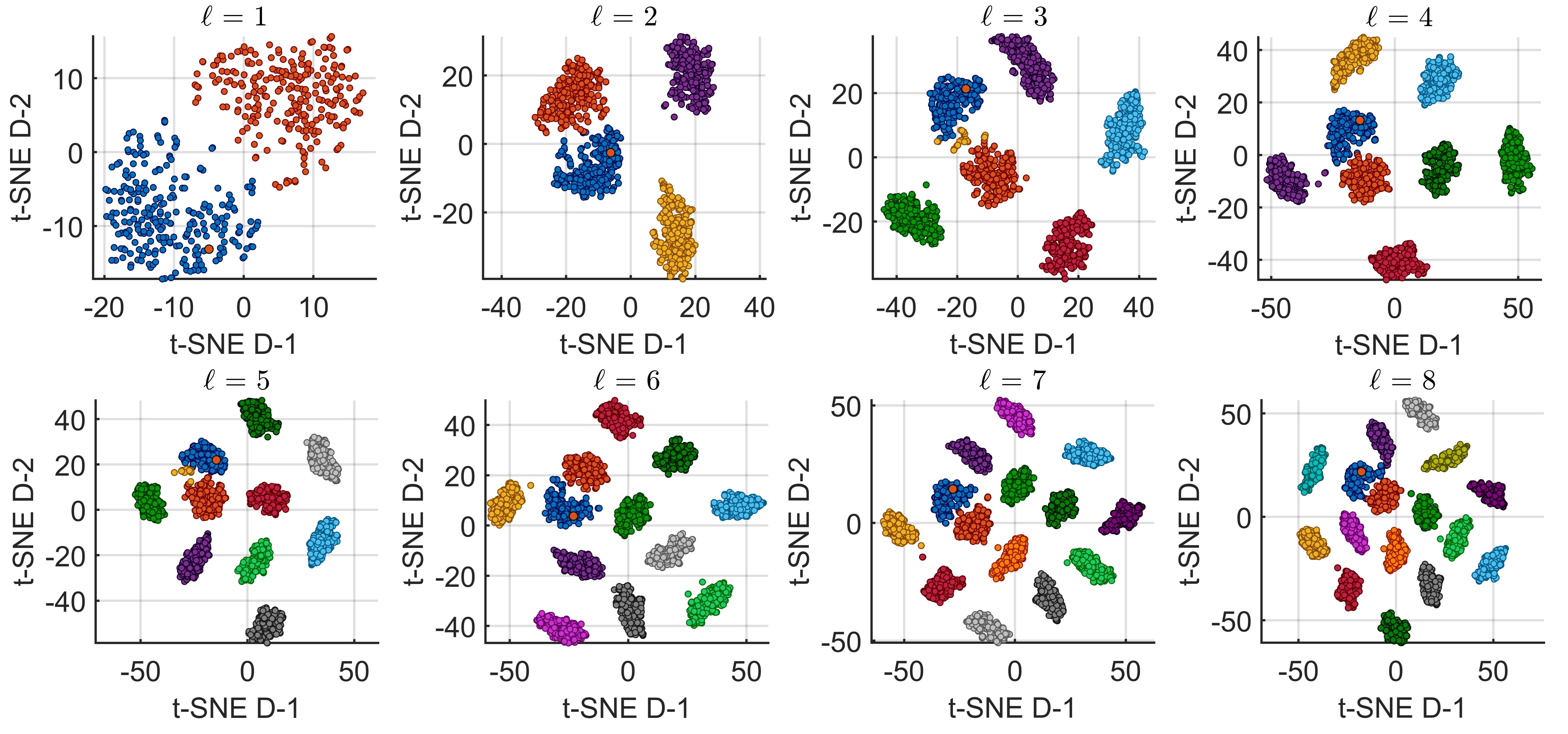}
  \caption{Cluster formation from random test samples \( w'^{(i)} \in \mathbb{R}^k \) (for \( i = 1, 2, \dots, N \)) is visualized in a lower-dimensional space using a 2D t-SNE plot. As the number of layers \( L \) increases, a total number of data points, denoted as \( LN \), are organized into $2L$ clusters. The model parameters are as follows: input dimension \( k = 15 \), network width \( n = 2^{25} \), and depth \( L = 8 \).}
  \label{fig:clus6f716}
\end{figure*}
\begin{figure*}[!htp]
\centering
  \includegraphics[scale=0.65]{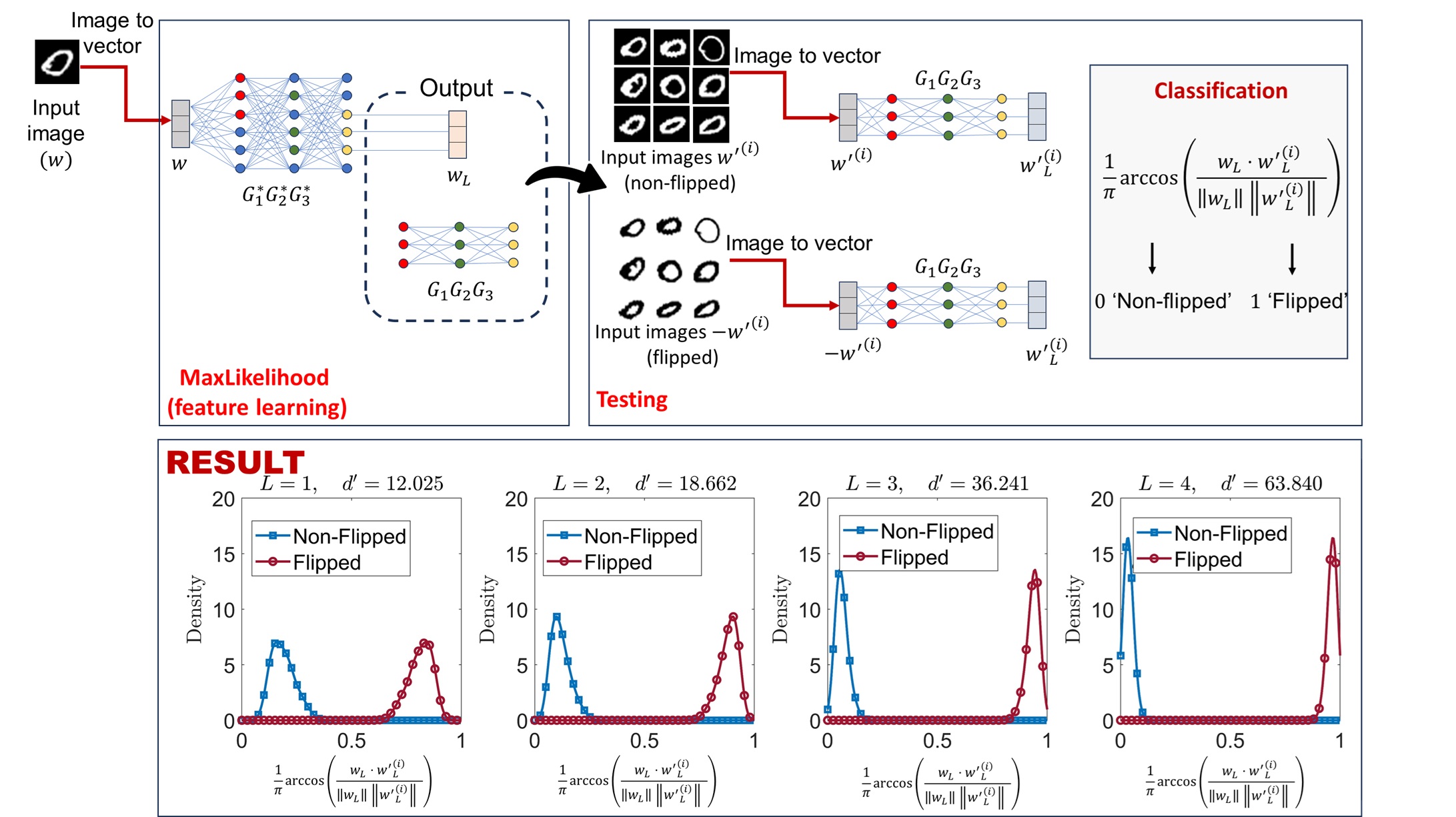}
  \caption{Flipped label experiment workflow and results demonstrating the improved classification performance with increasing network depth $L$, as measured by decidability index $d'$.  }
  \label{fig:frame2}
\end{figure*}

\textbf{Flipped Labels Classification}: Here, we demonstrate how the output model \( H_{\textnormal{Linear}}(\cdot) \) generalizes and enhances classification performance for arbitrarily flipped labels. This classification is based on the output feature vector \( w_L \), extracted during the unsupervised feature learning phase. The overall experimental procedure is illustrated in Figure~\ref{fig:frame2}.  

In this experiment, we utilize the MNIST dataset~\cite{deng2012mnist}, which consists of images of handwritten digits (0–9). Specifically, we focus on the single-digit ‘0’ scenario. The experiment comprises two main steps:  

\begin{enumerate}  
    \item \textit{Feature Learning}: The model is optimized using an input sample \( w \in \mathbb{R}^{k} \) and Algorithm~\ref{algo:maxlik}. The output model \( H_{\textnormal{Linear}}(\cdot) \) generates a feature vector \( w_L \), corresponding to a non-flipped label (e.g., digit ‘0’ with a black background).  

    \item \textit{Testing}: New feature vectors \( w'^{(i)}_{L} \) (\( i=1,2,\dots,N \)) are generated from previously unseen images using \( H_{\textnormal{Linear}}(\cdot) \). The objective is to classify each new feature vector \( w'^{(i)}_{L} \) as either belonging to a non-flipped or flipped digit.  
\end{enumerate}  

Figure~\ref{fig:frame2} illustrates the flipped label classification procedure, where we tested for \( L = 1,2,3,4 \) with network width \( n = 8000 \) and \( N = 3000 \) test samples, which were never seen by \( H_{\textnormal{Linear}}(\cdot) \) during the feature learning phase. For flipped label data, we apply sign flipping by multiplying each pixel value by \(-1\), effectively inverting the digit’s color (i.e., black to white) after normalization (dividing by the maximum grayscale value, 255). This produces an additional \( N = 3000 \) flipped test samples. Classification is performed by computing the pairwise angle difference between the output vector \( w_L \) and the testing vector \( w'^{(i)}_L \), given by:  
\begin{equation} \frac{1}{\pi} \arccos \left( \frac{w_L \cdot w'^{(i)}_L}{\|w_L\| \|w'^{(i)}_L\|} \right)\in (0,1).
\end{equation}  

The results shown in Figure~\ref{fig:frame2} highlight that, due to the collapse in the output feature space, as the network depth increases from \( L=1 \) to \( L=4 \), the two distributions (blue and red curves) shift in opposite directions. This demonstrates the model's ability to assert strong confidence in classifying the testing vector as either flipped or non-flipped, relative to the normalized output vector \( \frac{w_L}{\|w_L\|} \) extracted during the feature learning phase.  

To further quantify this improvement in classification confidence, we adopt the concept of \textit{decidability} from Daugman~\cite{daugman2009iris}. In a binary classification task (e.g., YES or NO decisions), the separation between two distributions—such as the red (flipped) and blue (non-flipped) curves in Figure~\ref{fig:frame2}—can be measured using the \textit{decidability index} \( d' \), defined as:  
\begin{equation}\label{eq:rr9}
d' = \frac{|\mu_1 - \mu_2|}{\sqrt{\frac{\sigma_1^2 + \sigma_2^2}{2}}},
\end{equation}  
where \( \mu_1, \mu_2 \) and \( \sigma_1, \sigma_2 \) represent the means and standard deviations of the two distributions, respectively.  

Clearly, the decidability score \( d' \) increases with increasing \( L \), indicating the output model \( H_{\textnormal{Linear}}(\cdot) \) generalizes to unseen images with improved confident in classification task. 

While most DNN architectures can classify flipped and non-flipped samples, our approach is unique in its unsupervised nature, relying on a single input sample for feature learning without gradient descent. This opens new avenues for studying DNN generalization behavior and offers new insights beyond conventional supervised learning frameworks.

\section{Convergence Analysis Within Extended Vector Space}
In this section, we deepen our analysis of the observed collapsing phenomenon by investigating its underlying causes. We formalize the representation of a feedforward linear neural network within an extended vector space, \( \mathbb{R}^{k+1} \), which incorporates information about each network layer. This framework enables us to trace the trajectory of an input vector \( w \in \mathbb{R}^{k} \) as it propagates through the network \( H_{\text{Linear}}(\cdot) \) with an arbitrary fixed width parameter \( n \). 

We introduce an additional dimension to represent the layer index \( \ell = 1, 2, \dots, L \), which indicates the position of the transformed vector at each layer. For example, consider an arbitrary input vector \( w = [x; y] \in \mathbb{R}^2 \). After passing through the first layer (i.e., \( \ell = 1 \)), we incorporate the layer index, representing the output vector as \( w_1 = [\ell c; x_1; y_1] \), where \( (x_1, y_1) \) denotes the coordinates of the transformed vector at layer \( \ell = 1 \) within the extended vector space \( \mathbb{R}^{k+1} \), and \( c > 0 \) is a constant scale. As illustrated in Fig.~\ref{fig:groupgigug} panel (a), this formalization allows us to visualize the trajectory of the input vector as it progresses through the layers, showing the coordinate changes at each layer.

Refer to Fig.~\ref{fig:groupgigug} panel (b). Appending \( \ell c \) to \( w_{\ell} \) yields the vector \( [\ell c; x_{\ell}; y_{\ell}] \). In the limit of large \( L \), for an arbitrary constant scale $c$, increasing the layer index $\ell=1,2,\ldots, L$ without optimization (i.e., norm maximization) causes the term \( \ell c \) to dominate the norm $\norm{w_{\ell}}=\norm{[x_{\ell};y_{\ell}]}$. The normalization of this vector in \( \mathbb{R}^{k+1} \) can be approximated as:

\begin{align}\label{eq:convffagsol29ww}
\frac{[\ell c; x_\ell; y_\ell]}{\|[\ell c; x_\ell; y_\ell]\|} \approx \frac{1}{\|\ell c\|} \left[ \ell c; 0; 0 \right] = [1; 0; 0] \in \mathbb{R}^{k+1}.
\end{align}
On the other hand, when the optimization process is applied to maximize \( \| w_{\ell} \| \) using the MaxLikelihood algorithm (Algorithm~\ref{algo:maxlik})—ensures that \( \| w_{\ell} \| \) dominates over \( \ell c \). This optimization drives the angle \( \theta \) to approach zero, where \( \theta \) represents the angle relationship between output vector \( w_L \) and the newly introduced layer index axis. Consequently, \( w_L \) aligns approximately with the coordinate \( \frac{[0; x_L; y_L]}{\|[x_L; y_L]\|} \) within \( x \)-\( y \) plane at the final layer \( \ell = L \).

We present the following theorem to formally explain the underlying reason for the observed convergence behavior of the output model \( H_{\text{Linear}}(\cdot) \), within the extended vector space \( \mathbb{R}^{k+1} \), where the angle \( \theta \) is interpreted as the inferred parameter.

\begin{theorem}\label{th:1}
Let \( w \in \mathbb{R}^{k} \) and \( w' \in \mathbb{R}^{k} \) be two random vectors. As the network depth \( L \) of \( H_{\text{Linear}}(\cdot) \) becomes sufficiently large (i.e., overparameterization for a fixed width $n>k$), there exists an approximate solution to the normalized vectors:

\begin{align}
\frac{[0; w_{L}]}{\|w_{L}\|} = \frac{[0; H_{\textnormal{Linear}}(w)]}{\|H_{\textnormal{Linear}}(w)\|} \in \mathbb{R}^{k+1},
\end{align}
and
\begin{align}
\frac{[0; w'_{L}]}{\|w'_{L}\|} = \frac{[0; H_{\textnormal{Linear}}(w')]}{\|H_{\textnormal{Linear}}(w')\|} \in \mathbb{R}^{k+1},
\end{align}
where the angle between these vectors converges to zero asymptotically, as described by
\begin{align}\label{eq:refhuo9022}
\boxed{\theta = \bigg[\theta_0 = \frac{1}{\pi} \arccos\left( \frac{[0; w_{L}]}{\|w_{L}\|} \cdot \frac{[0; w'_{L}]}{\|w'_{L}\|} \right) \bigg]\rightarrow 0,}
\end{align}
yielding the opposite solution:
\begin{align}\label{eq:r4555huo9g3022}
\boxed{1 - \theta =\bigg[ 1 - \theta_0 = 1 - \frac{1}{\pi} \arccos\left( \frac{[0; w_{L}]}{\|w_{L}\|} \cdot \frac{[0; w'_{L}]}{\|w'_{L}\|} \right) \bigg]\rightarrow 1.}
\end{align}

\end{theorem}

\begin{proof}
Consider the output vector \( w_{\ell} = H_{\text{Linear}}(w) = G_{\ell} G_{\ell-1} \dots G_{1}(w) \) at a particular layer \( \ell \leq L \), represented with the layer index appended as \( w_{\ell} =[\ell c; w_{\ell}]= [\ell c; x_\ell; y_\ell] \in \mathbb{R}^{k+1} \). By normalizing these vectors, we obtain the following expression for the normalized output at layer \( \ell \):

\begin{align}
&\frac{[\ell c; x_\ell; y_\ell]}{\|[\ell c; x_\ell; y_\ell]\|} \nonumber \\
&= \left[\underbrace{\frac{\ell c}{\sqrt{(\ell c)^2 + \|w_\ell\|^2}}}_{\textnormal{Layer information}}; \underbrace{\frac{x_\ell}{\sqrt{(\ell c)^2 + \|w_\ell\|^2}}; \frac{y_\ell}{\sqrt{(\ell c)^2 + \|w_\ell\|^2}}}_{\textnormal{Spatial information}}\right].
\end{align}
As the depth \( L \) of the network increase, the norm of the weight vector \( \|w_\ell\| \) grows as \( \ell \to L \). In the limit of large \( L \), the contribution of the normalized layer index term \( \frac{\ell c}{\sqrt{(\ell c)^2+\norm{w_{\ell}}^2}} \) becomes negligible, and the normalized vector approaches the following form:

\begin{align}\label{eq:convgsol29ww}
\frac{[\ell c; x_\ell; y_\ell]}{\|[\ell c; x_\ell; y_\ell]\|} \approx \frac{1}{\|w_\ell\|} \left[ 0; x_\ell; y_\ell \right] \in \mathbb{R}^{k+1}.
\end{align}
This approximation holds for any input vector, even for random unseen test vector \( w' \in \mathbb{R}^k \).  As \( \ell \rightarrow L \), in the limit of large \( L \), the angle \( \theta \) of the normalized vectors $\frac{[\ell c;w_{L}]}{\|[\ell c;w_{L}]\|} \quad \text{and} \quad \frac{[\ell c;w'_{L}]}{\|[\ell c;w'_{L}]\|},$
measured relative to the \( x \)-\( y \) plane, converges to zero, i.e., \( \theta \to 0 \).

With the constraints specified in Proposition~\ref{propo:1} (i.e., \( \theta = \theta_0 \)) and Proposition~\ref{eq:proposition13} (i.e., \( \theta_0 = \frac{1}{\pi} \arccos\left( \frac{w_{L}}{\|w_{L}\|} \cdot \frac{w'_{L}}{\|w'_{L}\|} \right) \)), we yield Eq.~\eqref{eq:refhuo9022} since $  {(\frac{w'_L}{\|w'_L\|} \cdot \frac{w_L}{\|w_L\|}) }={( \frac{[0; w'_L]}{\|w'_L\|} \cdot \frac{[0; w_L]}{\|w_L\|} )}\rightarrow 0$.

Normalize \( \theta \in (0, \frac{1}{2}) \) by dividing it by \( \pi \), resulting in a convergence behavior such that \( \theta \to 0 \) holds for all \( \theta = \theta_0 < \frac{1}{2} \). Conversely, for \( (1 - \theta) \in (\frac{1}{2}, 1) \), the opposite behavior holds: \( (1-\theta) = (1-\theta_0) > \frac{1}{2} \), where \( (1 - \theta) \to 1 \) as \( \theta \to 0 \).

This symmetric convergence behavior of the angle parameter \( \theta \to 0 \) (or \((1- \theta) \to 1 \)) plays a critical role in the collapsing phenomenon observed within the overparameterized regime, consistent with Observation~\ref{claim:1}.

\end{proof}

\textbf{Remark 1:} The proof of Theorem~\ref{th:1} extends the neural network's output feature space to \( \mathbb{R}^{k+1} \) by incorporating the layer index. This extension enables the layer index to influence the trajectory of the feature vectors as they propagate through the network, particularly after the vectors are normalized to unit norm. The approach is conceptually inspired by Einstein’s unification of space and time into spacetime, which explains intriguing phenomena such as time dilation and length contraction. In our framework, the layer index serves as a temporal-like dimension, shaping the evolution of feature vectors as a trajectory in the extended space. The angle parameter $\theta$ acts as a constraint governing this evolution, reflecting the interdependence between the layer index and the spatial coordinates of the feature vectors after normalization. This formulation facilitates convergence in overparameterized networks through maximum likelihood optimization. Additionally, although our analysis focuses on 2D input vectors for geometric visualization in the extended 3D space, the approach generalizes to \( k > 2 \). In this case, the input vector \( w_{\ell} \) (or \( w'_{\ell} \)) is viewed as a manifold embedded in \( \mathbb{R}^{k+1} \). As the network's width \( n \) and depth \( L \) increase, normalization reduces the layer index influence, causing the feature vectors in the final layer to align along the \( \mathbb{R}^k \) manifold, approximately, as reflected in pairwise distance measures through inner product.

\begin{figure*}[!htp]
\centering
  \includegraphics[scale=0.73]{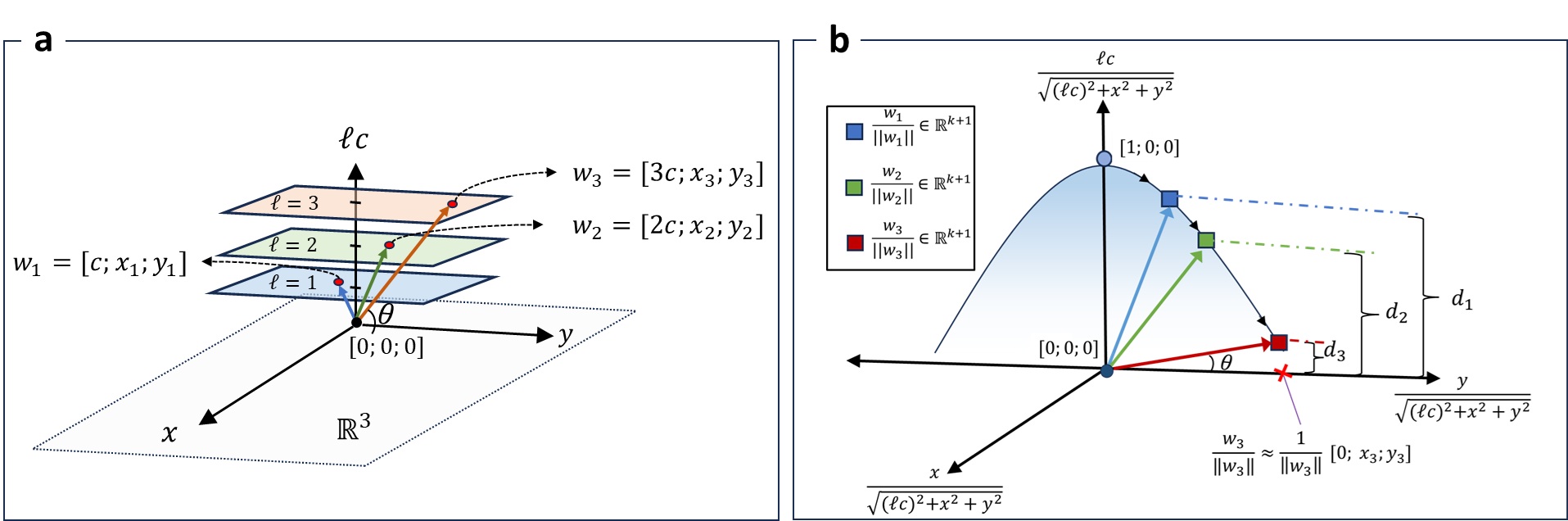}
  \caption{
(a) Coordinates representing the output vector with length \( k = 2 \) (2D) in a feedforward neural network with \( L = 3 \), depicted within an extended vector space \(\mathbb{R}^{k+1}\) (3D), where an additional dimension is included to represent the layer index. (b) Illustration of the convergence of the angle \( \theta \rightarrow 0 \) in the (normalized) extended vector space. This behavior is driven by the term \( d_{\ell} = \frac{\ell c}{\sqrt{(\ell c)^2 + x_\ell^2 + y_\ell^2}} \to 0 \) as a result of maximizing $\norm{w_{\ell}}=\norm{[x_{\ell};y_{\ell}]}$ and normalization, which is crucial for deriving the wormhole solution as established in Theorem~\ref{thr:wormhole}.
}
  \label{fig:groupgigug}
\end{figure*}

\section{Wormhole Enabled Shortcut Learning}\label{sec:weomhtot}

\begin{figure*}[!htp]
\centering
  \includegraphics[scale=0.35]{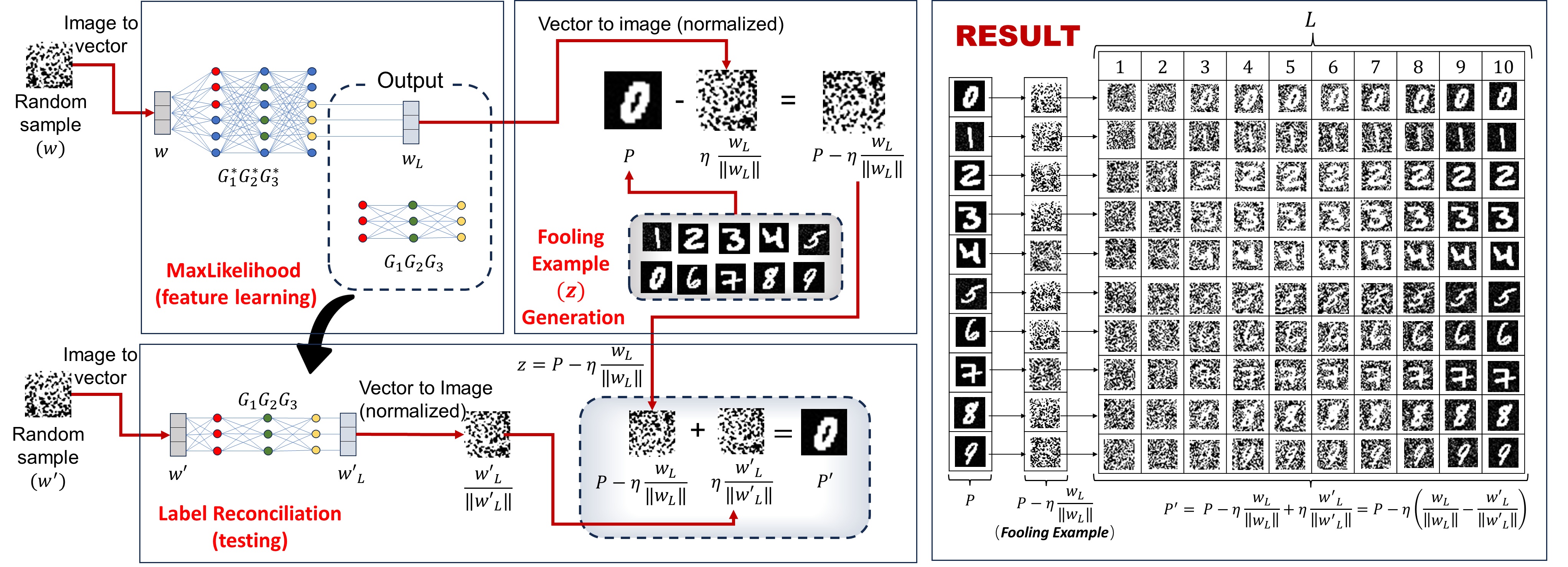}
  \caption{Label reconciliation workflow and results demonstrating the existence of a wormhole even with large perturbations applied to the input image $P$, where the perturbation scale $\eta=10000$. All input samples $w$ and $w'$ are randomly sampled from a Gaussian distribution with mean 0 and variance 1.}
  \label{fig:wormholeworkflow}
\end{figure*}
In the context of feature learning, when the angle between output vectors \( w_L \) and \( w'_L \), corresponding to distinct inputs \( w \) and \( w' \), approaches zero, the network enters a state that we term degeneracy, meaning it can no longer differentiate between different inputs. In this degenerate state, modifications to the network parameters \( \theta \) no longer influence the output, and the loss function or gradient evaluates to zero—even for inputs composed of random noise \cite{zhang2021understanding}. This behavior underscores the network's tendency to memorize inputs, a phenomenon commonly observed in overparameterized models. Consequently, the network may end up learning trivial solutions that do not capture the true complexity of the data, thereby leading to overfitting and poor generalization to unseen examples.

We investigate this phenomenon and demonstrate that, within a feedforward linear neural network, shortcuts akin to "wormholes" can form. These shortcuts bridge distant inputs in the extended vector space \( \mathbb{R}^{k+1} \) and can bypass the degeneracy, preserving the model's ability to generalize even when exposed to apparent noisy inputs, such as fooling examples.

In fact, the notion of wormholes in DNNs has been explored previously. For instance, Gaziv et al.~\cite{gaziv2023robustified} present compelling evidence of low-norm image perturbations disrupting human categorization behavior, particularly in natural images. Their findings indicate that such perturbations create wormholes between perceptual categories, facilitating transitions between semantically distant regions in image space.

While these wormholes have been characterized for low-norm perturbations, questions remain regarding their underlying formation mechanisms and persistence under large-norm perturbations. Large-norm perturbations, such as those observed in fooling examples, may resemble random noise to human perception.

Here, we extend the analysis by formalizing Theorem~\ref{thr:wormhole}, which demonstrates that, under an overparameterized regime, a feedforward neural network \( H_{\textnormal{Linear}}(\cdot) \) can effectively create a shortcut, referred to as a "wormhole," between an arbitrary pair of distant sample points \( P = -\eta \frac{[0; w_{L}]}{\|w_{L}\|} \in \mathbb{R}^{k+1} \) and \( P' = \eta \frac{[0; w'_{L}]}{\|w'_{L}\|} \in \mathbb{R}^{k+1} \). This holds even for an arbitrarily large scale value \( \eta > 0 \), which generates a substantial distance between \( P \) and \( P' \).

\begin{theorem}\label{thr:wormhole}
For arbitrary vectors \( w, w' \in \mathbb{R}^{k} \) and a large scale \( \eta > 0 \), let \( P = -\eta \frac{[0; w_L]}{\|w_L\|} \in \mathbb{R}^{k+1} \) be a point located within the extended vector space \( \mathbb{R}^{k+1} \), where \( w_L = H_{\textnormal{Linear}}(w) \) is the output of a feedforward linear network \( H_{\textnormal{Linear}}(\cdot) \) optimized by Algorithm~\ref{algo:maxlik}. Then, there exists another point \( P' = \eta \frac{[0; w'_L]}{\|w'_L\|} \in \mathbb{R}^{k+1} \), with \( w'_L = H_{\textnormal{Linear}}(w') \), such that in the limit of large enough network depth $L$ (i.e., overparameterized regime with the width $n>k$ fixed):
\begin{align}
\|P - P'\| \to 0,
\end{align}
forming an effective "wormhole" between \( P \) and \( P' \).
\end{theorem}

\begin{proof}
For large enough $L$, utilizing the model \( H_{\textnormal{Linear}}(\cdot) \) optimized by Algorithm~\ref{algo:maxlik}, one obtains an approximation of the unit vector represented as:
\begin{align}
\frac{[0; w_L]}{\|w_L\|} = \frac{[0; H_{\textnormal{Linear}}(w)]}{\|H_{\textnormal{Linear}}(w)\|} \in \mathbb{R}^{k+1},
\end{align}
where \( w \in \mathbb{R}^k \) is mapped into \( \mathbb{R}^{k+1} \). This allows us to represent any arbitrary point \( P \) scaled along the unit vector as:
\begin{align}\label{wom234}
P = -\eta \frac{[0; w_L]}{\|w_L\|} \implies P + \eta \frac{[0; w_L]}{\|w_L\|} = 0.
\end{align}
Similarly, for \( P' \), which corresponds to another unit vector \( \frac{[0; w'_L]}{\|w'_L\|} \), we have:
\begin{align}\label{wom23}
P' = \eta \frac{[0; w'_L]}{\|w'_L\|} \implies P' - \eta \frac{[0; w'_L]}{\|w'_L\|} = 0.
\end{align}
From Eqs.~\eqref{wom234} and \eqref{wom23}, we obtain:
\begin{align}\label{eq:wo231m23}
P' - P = \eta \left( \frac{[0; w'_L]}{\|w'_L\|} + \frac{[0; w_L]}{\|w_L\|} \right).
\end{align}
This implies that \( \|P' - P\| \to 0 \) if $\left\| \frac{[0; w'_L]}{\|w'_L\|} + \frac{[0; w_L]}{\|w_L\|} \right\| \to 0$, 
which is achievable in the limit of sufficiently large \( L \), with an arbitrary scale \( \eta > 0 \), in the overparameterized regime. In this regime, we have:
\begin{align} \label{eq:converji93}
\norm{\frac{w'_L}{\|w'_L\|} + \frac{w_L}{\|w_L\|}} = \left\| \frac{[0; w'_L]}{\|w'_L\|} + \frac{[0; w_L]}{\|w_L\|} \right\| \to 0,
\end{align}
enabling a state of perfect anti-correlation between \( \frac{w'_L}{\|w'_L\|} \) and \( \frac{w_L}{\|w_L\|} \) within a lower-dimension vector space $\RR^{k}$, as described in Theorem~\ref{th:1} and Eq.~\eqref{eq:23antico4}.
\end{proof}

%It is also worth to highlight that by denoting the another loss function
%\begin{equation}
%\small
%\mathcal{L}({\theta},w,w') = \frac{H_{\textnormal{Linear}}(w')}{\norm{H_{\textnormal{Linear}}(w')}} + \frac{H_{\textnormal{Linear}}(w)}{\norm{H_{\textnormal{Linear}}(w)}} = \frac{{w}'_{L}}{\norm{{w}'_{L}}} + \frac{{w}_{L}}{\norm{{w}_{L}}}.
%\end{equation}
%Eq. \eqref{eq:wo231m23} can be transform to approximate
%\begin{align}\label{eq:goodfelow2}
%P'={P+\eta \sign(\nabla_{w}\mathcal{L}({\theta},w,w')}).
%\end{align}
%This demonstrates how the solution derived in Eq. \eqref{eq:wo231m23} under perfect anti-correlation s.t. \( \frac{w_L}{\|w_L\|} \approx -\frac{w'_L}{\|w'_L\|} \) can be utilized to establish the existence of a solution to Eq. \eqref{eq:goodfelow2} for \( P' \approx P \),
% which corresponds to the Fast Gradient Sign Method (FGSM) proposed by Goodfellow et al. \cite{goodfellow2014explaining}. It is important to note that FGSM is used to generate adversarial examples, which involve introducing perturbations to \( P \) that are imperceptible to the human eye (i.e., when \( \eta \) is small), yet are capable of causing misclassification by a model. Thus, the above derivation also provides an explanation for the existence of adversarial examples in a model, attributing their creation to the presence of "wormholes" connecting the perturbed input to the original, unperturbed input.

\subsection{Label Reconciliation for Fooling Example via Wormhole}
With the previously established existence of a wormhole for an arbitrarily large scaling factor $\eta$, we now aim to validate the possibility of associating a meaningful label with a fooling example. Our main intuition is to generate a fooling example and examine whether any solution within it corresponds to a meaningful label from a human perception perspective. To formalize this, we introduce the concept of label reconciliation:

\begin{definition}\label{def:12fg}
Label reconciliation is the process of recovering a meaningful label $P$, consistent with human perception, given a randomly generated label $P'$ and a fooling example $z$. The goal is to find a mapping $M$ such that $M(P', z) = P$.
\end{definition}

\textbf{Label Reconciliation Setup:} In this work, we do not generate fooling examples as pure random noise as verifying the existence of a solution within pure random noise would be computationally intractable. Instead, we conduct the label reconciliation experiment using Algorithm \ref{algo:maxlik}, which comprises three main processes described as follows:

\begin{enumerate}
    \item \textit{Feature Learning:} This step involves optimizing a model using an arbitrary random sample $w \in \mathbb{R}^{k}$. The model generates an output vector $\eta \frac{w_{L}}{\norm{w_{L}}}$ used to introduce input perturbation (noise), where $\eta > 0$ is a scaling factor.
    \item \textit{Fooling Example Generation:} This step generates fooling examples $z = P - \eta \frac{w_{L}}{\norm{w_{L}}}$ by injecting noise, represented by the generated output vector \(\eta \frac{w_{L}}{\norm{w_{L}}}\), into a meaningful label $P$. A large perturbation $\eta$ introduces significant distortion into the meaningful label $P$, making it appear as random noise from a human perception standpoint.
    \item \textit{Label Reconciliation:} This step aims to recover the meaningful label $P$ from the random label, defined as $P'=z+\eta\frac{w'_{L}}{\norm{w'_{L}}}$, where $\eta\frac{w'_{L}}{\norm{w'_{L}}}$ is represents another random output vector.
\end{enumerate}

\textbf{Results and Discussion:} Fig. \ref{fig:wormholeworkflow} illustrates the experimental workflow for label reconciliation using the open-source MNIST dataset \cite{deng2012mnist}. Each image is converted into a single vector of dimension \( 28 \times 28 = 784 \) as the meaningful label $P$ (represented as "digits") for fooling example generation. The parameter \(\eta\) is set to \(10000\) throughout the experiment.

On the right-hand side of Fig. \ref{fig:wormholeworkflow}, the results show the successful reconciliation of the original label $P$ using $P'$ and the fooling example $z$ as the network depth $L$ increases. This result can be explained by the establishment of the wormhole as stated in Theorem \ref{thr:wormhole} (when $z\rightarrow 0)$. Specifically, a wormhole solution exists for $\|P - P'\| \to 0,$ even for an arbitrarily large $\eta$ leading to a state of perfect anti-correlation as described in Eq. \eqref{eq:converji93}. Therefore, label reconciliation can be achieved through such a wormhole, where:
\begin{align}\label{eq:dsjfiofg89}
P = M(P', z) = P' - \eta \underbrace{\left( \frac{w'_L}{\|w'_L\|} - \frac{w_L}{\|w_L\|} \right)}_{\rightarrow  0},
\end{align}
implies a solution of perfect correlation where $\frac{w'_L}{\|w'_L\|} = \frac{w_L}{\|w_L\|}$, consistent with Observation \ref{claim:1}.

It is noteworthy that the convergence of the distances between normalized output vectors to arbitrarily small values, as delineated in Eq. \eqref{eq:dsjfiofg89}, does not necessarily lead to a state of degeneracy or overfitting that would compromise the model's generalization capacity. Rather, this convergence can be interpreted as indicative of the model's inherent simplicity bias, whereby the network effectively employs a "wormhole" shortcut to approximate a solution.  This mechanism enables the network to reconcile the meaningful label \( P \) with a random  label \( P' \), ultimately, facilitates the derivation of
$P$ using $P'$.

To better relate these findings to the broader success of gradient-based optimization, consider the loss function defined as: \( \mathcal{L}(\theta, y, w') = \left\| \frac{w'_L}{\|w'_L\|} - y \right\| \), where we describe $y=\frac{w_L}{\|w_L\|}$ as a vector associated with the label $P$ that is meaningful to humans (e.g., a label such as digit "zero" or "one"). Given the fooling example and the existence of a solution as expressed in Eq.~\eqref{eq:dsjfiofg89}, this formulation enables gradient-based technique (i.e., gradient descent) to transform a random label \( P' \) into a meaningful label \( P \) follows
\begin{align}\label{eq:dsefffee9}
P \leftarrow P' - \eta \underbrace{\nabla_{\theta} \mathcal{L}(\theta, y, w')}_{\rightarrow 0},
\end{align}
where a stationary point, such that \( \nabla_{\theta} \mathcal{L}(\theta, y, w') = 0 \), would yield a deterministic solution for \( P = P' \). 

It is also worth noting that \( w' \) can be chosen as any sample meaningful to humans, such as an image of a "Panda." In this case, the corresponding output vector \( \frac{w'_{L}}{\|w'_{L}\|} \) becomes associated with the "Panda" label for \( P' \).

By taking the sign of the computed gradient, Eq. \eqref{eq:dsefffee9} can be transformed into 
\begin{align}
P' \leftarrow P + \eta \cdot \mathsf{sgn}({\nabla_{\theta} \mathcal{L}(\theta, y, w'))},
\end{align}
which represents the Fast-Gradient Sign Method (FGSM) proposed by Goodfellow \cite{goodfellow2014explaining} for adversarial example generation ("Panda" $\leftarrow$ "digits") when $\eta$ is small, i.e. $P\approx P'$. In light of this, above derivation also provides a compelling explanation for the existence of adversarial examples, attributed to the wormhole solution.

\textbf{Remark 2}: Although our work primarily focuses on unsupervised learning, our findings on label reconciliation have broader implications for supervised learning, where meaningful labels are typically considered essential for model training. Specifically, we demonstrate the existence of a mapping \( M(z, P') = P \),  facilitated by the newly derived wormhole solution (see Eq.~(\ref{eq:dsjfiofg89})), which enables the output model \( H_{\textnormal{Linear}}(.) \) to transform an arbitrary label \( P' \)—which may appear random to human perception—into a meaningful label \( P \).  Without confirming this mapping exists, there is no guarantee that a model can reliably transform \( P' \) into \( P \) using gradient-based methods (see Eq.~(\ref{eq:dsefffee9})). In other words, if label reconciliation is infeasible, a fooling example may remain adversarial rather than aligning with a meaningful label, making it indistinguishable from structured noise. This, in turn, raises fundamental questions about the role of label reconciliation in supervised learning, particularly regarding the conditions under which meaningful label alignment can be guaranteed.

\section{Conclusion}
This study explores the fundamental mechanisms that enable deep neural networks (DNNs) to generalize from unstructured inputs. Through analytical derivations and convergence analyses, we demonstrate that feedforward neural networks can achieve maximum likelihood estimation (MLE) in unsupervised settings, without reliance on labeled data and numerical optimization techniques such as stochastic gradient descent (SGD). Our results reveal that DNNs generalize to arbitrary input data by maximizing input norms under normalization constraints. Furthermore, we demonstrate that overparameterized networks exhibit feature space collapse, a phenomenon that intrinsically enables effective data clustering via memorization.

A key insight of this work is the identification of the "wormhole" solution, which bypasses feature collapse under extreme conditions, offering a novel perspective on how DNNs map noisy inputs to coherent outputs. This result also provides a theoretical basis for understanding shortcut learning and how DNNs generalize to fooling examples.

While our analyses are grounded in feedforward linear networks—key components of more complex DNN architectures—we recognize that extending these findings to models such as convolutional networks or transformer-based architectures poses challenges that warrant further investigation. For example, our flipped label classification experiments (see Figure \ref{fig:frame2}) illustrate how processing an input image \( w \) through \( H_{\textnormal{Linear}}(\cdot) \) to extract the output feature vector \( w_L \) can shed light on these dynamics. Future research might integrate additional mechanisms (e.g., graph-based methods, Fourier transformations, or attention modules) prior to feature embedding within \( H_{\textnormal{Linear}}(\cdot) \).  Such extensions could deepen our understanding of the collapse phenomenon across diverse feature extraction strategies and possibly inform the design of more interpretable, efficient DNN architectures.

Despite these limitations, our findings resonate with insights from cognitive science, where human learners demonstrate remarkable generalization from limited data—a capability still elusive for many existing AI systems. We propose that a deeper exploration of the intrinsic learning dynamics in unsupervised and contrastive learning settings may pave the way for developing more efficient, human-like learning systems. Shifting the focus from merely scaling model parameters and data toward understanding learning dynamics holds significant promise for advancing research in artificial general intelligence (AGI).

\bibliographystyle{IEEEtran}
\bibliography{tnnlsbib}

% Generated by IEEEtran.bst, version: 1.14 (2015/08/26)
\begin{thebibliography}{10}
\providecommand{\url}[1]{#1}
\csname url@samestyle\endcsname
\providecommand{\newblock}{\relax}
\providecommand{\bibinfo}[2]{#2}
\providecommand{\BIBentrySTDinterwordspacing}{\spaceskip=0pt\relax}
\providecommand{\BIBentryALTinterwordstretchfactor}{4}
\providecommand{\BIBentryALTinterwordspacing}{\spaceskip=\fontdimen2\font plus
\BIBentryALTinterwordstretchfactor\fontdimen3\font minus
  \fontdimen4\font\relax}
\providecommand{\BIBforeignlanguage}[2]{{%
\expandafter\ifx\csname l@#1\endcsname\relax
\typeout{** WARNING: IEEEtran.bst: No hyphenation pattern has been}%
\typeout{** loaded for the language `#1'. Using the pattern for}%
\typeout{** the default language instead.}%
\else
\language=\csname l@#1\endcsname
\fi
#2}}
\providecommand{\BIBdecl}{\relax}
\BIBdecl

\bibitem{spears2018deep}
B.~K. Spears, J.~Brase, P.-T. Bremer, B.~Chen, J.~Field, J.~Gaffney, M.~Kruse,
  S.~Langer, K.~Lewis, R.~Nora \emph{et~al.}, ``Deep learning: A guide for
  practitioners in the physical sciences,'' \emph{Physics of Plasmas}, vol.~25,
  no.~8, 2018.

\bibitem{rudin2019stop}
C.~Rudin, ``Stop explaining black box machine learning models for high stakes
  decisions and use interpretable models instead,'' \emph{Nature machine
  intelligence}, vol.~1, no.~5, pp. 206--215, 2019.

\bibitem{zheng2019finbrain}
X.-l. Zheng, M.-y. Zhu, Q.-b. Li, C.-c. Chen, and Y.-c. Tan, ``Finbrain: when
  finance meets ai 2.0,'' \emph{Frontiers of Information Technology \&
  Electronic Engineering}, vol.~20, no.~7, pp. 914--924, 2019.

\bibitem{franzke2021data}
A.~S. Franzke, I.~Muis, and M.~T. Sch{\"a}fer, ``Data ethics decision aid
  (deda): a dialogical framework for ethical inquiry of ai and data projects in
  the netherlands,'' \emph{Ethics and Information Technology}, vol.~23, no.~3,
  pp. 551--567, 2021.

\bibitem{zhang2021understanding}
C.~Zhang, S.~Bengio, M.~Hardt, B.~Recht, and O.~Vinyals, ``Understanding deep
  learning (still) requires rethinking generalization,'' \emph{Communications
  of the ACM}, vol.~64, no.~3, pp. 107--115, 2021.

\bibitem{keskar2016large}
N.~S. Keskar, J.~Nocedal, P.~T.~P. Tang, D.~Mudigere, and M.~Smelyanskiy, ``On
  large-batch training for deep learning: Generalization gap and sharp
  minima,'' in \emph{The 5th International Conference on Learning
  Representations (ICLR)}, 2017.

\bibitem{dauphin2013big}
Y.~N. Dauphin and Y.~Bengio, ``Big neural networks waste capacity,'' in
  \emph{ICLR (Workshop)}, 2013.

\bibitem{bengio1994learning}
Y.~Bengio, P.~Simard, and P.~Frasconi, ``Learning long-term dependencies with
  gradient descent is difficult,'' \emph{IEEE transactions on neural networks},
  vol.~5, no.~2, pp. 157--166, 1994.

\bibitem{blum1988training}
A.~Blum and R.~Rivest, ``Training a 3-node neural network is np-complete,''
  \emph{Advances in neural information processing systems}, vol.~1, 1988.

\bibitem{arora2018convergence}
S.~Arora, N.~Golowich, N.~Cohen, and W.~Hu, ``A convergence analysis of
  gradient descent for deep linear neural networks,'' in \emph{The 7th
  International Conference on Learning Representations, (ICLR)}, 2019.

\bibitem{bah2022learning}
B.~Bah, H.~Rauhut, U.~Terstiege, and M.~Westdickenberg, ``Learning deep linear
  neural networks: Riemannian gradient flows and convergence to global
  minimizers,'' \emph{Information and Inference: A Journal of the IMA},
  vol.~11, no.~1, pp. 307--353, 2022.

\bibitem{nguegnang2021convergence}
G.~M. Nguegnang, H.~Rauhut, and U.~Terstiege, ``Convergence of gradient descent
  for learning linear neural networks,'' \emph{Advances in Continuous and
  Discrete Models}, vol. 2024, no.~1, p.~23, 2024.

\bibitem{saxe2013exact}
A.~Saxe, J.~McClelland, and S.~Ganguli, ``Exact solutions to the nonlinear
  dynamics of learning in deep linear neural networks,'' in \emph{Proceedings
  of the International Conference on Learning Represenatations (ICLR)}, 2014.

\bibitem{arora2018optimization}
S.~Arora, N.~Cohen, and E.~Hazan, ``On the optimization of deep networks:
  Implicit acceleration by overparameterization,'' in \emph{International
  conference on machine learning}.\hskip 1em plus 0.5em minus 0.4em\relax PMLR,
  2018, pp. 244--253.

\bibitem{arora2019implicit}
S.~Arora, N.~Cohen, W.~Hu, and Y.~Luo, ``Implicit regularization in deep matrix
  factorization,'' \emph{Advances in Neural Information Processing Systems},
  vol.~32, 2019.

\bibitem{chou2024gradient}
H.-H. Chou, C.~Gieshoff, J.~Maly, and H.~Rauhut, ``Gradient descent for deep
  matrix factorization: Dynamics and implicit bias towards low rank,''
  \emph{Applied and Computational Harmonic Analysis}, vol.~68, p. 101595, 2024.

\bibitem{neyshabur2017exploring}
B.~Neyshabur, S.~Bhojanapalli, D.~McAllester, and N.~Srebro, ``Exploring
  generalization in deep learning,'' \emph{Advances in neural information
  processing systems}, vol.~30, 2017.

\bibitem{pesme2021implicit}
S.~Pesme, L.~Pillaud-Vivien, and N.~Flammarion, ``Implicit bias of sgd for
  diagonal linear networks: a provable benefit of stochasticity,''
  \emph{Advances in Neural Information Processing Systems}, vol.~34, pp.
  29\,218--29\,230, 2021.

\bibitem{frei2022implicit}
S.~Frei, G.~Vardi, P.~L. Bartlett, and W.~Hu, ``Implicit bias in leaky relu
  networks trained on high-dimensional data,'' in \emph{The 11th International
  Conference on Learning Representations (ICLR)}, 2023.

\bibitem{arpit2017closer}
D.~Arpit, S.~Jastrzebski, N.~Ballas, D.~Krueger, E.~Bengio, M.~S. Kanwal,
  T.~Maharaj, A.~Fischer, A.~Courville, Y.~Bengio \emph{et~al.}, ``A closer
  look at memorization in deep networks,'' in \emph{Proceedings of the 34th
  International Conference on Machine Learning-Volume 70}, 2017, pp. 233--242.

\bibitem{kalimeris2019sgd}
D.~Kalimeris, G.~Kaplun, P.~Nakkiran, B.~Edelman, T.~Yang, B.~Barak, and
  H.~Zhang, ``Sgd on neural networks learns functions of increasing
  complexity,'' \emph{Advances in neural information processing systems},
  vol.~32, 2019.

\bibitem{valle2018deep}
G.~De~Palma, B.~Kiani, and S.~Lloyd, ``Random deep neural networks are biased
  towards simple functions,'' \emph{Advances in Neural Information Processing
  Systems}, vol.~32, 2019.

\bibitem{geirhos2020shortcut}
R.~Geirhos, J.-H. Jacobsen, C.~Michaelis, R.~Zemel, W.~Brendel, M.~Bethge, and
  F.~A. Wichmann, ``Shortcut learning in deep neural networks,'' \emph{Nature
  Machine Intelligence}, vol.~2, no.~11, pp. 665--673, 2020.

\bibitem{hermann2020shapes}
K.~Hermann and A.~Lampinen, ``What shapes feature representations? exploring
  datasets, architectures, and training,'' \emph{Advances in Neural Information
  Processing Systems}, vol.~33, pp. 9995--10\,006, 2020.

\bibitem{scimeca2021shortcut}
L.~Scimeca, S.~J. Oh, S.~Chun, M.~Poli, and S.~Yun, ``Which shortcut cues will
  {DNN}s choose? a study from the parameter-space perspective,'' in \emph{The
  10th International Conference on Learning Representations (ICLR)}, 2022.

\bibitem{papyan2020prevalence}
V.~Papyan, X.~Han, and D.~L. Donoho, ``Prevalence of neural collapse during the
  terminal phase of deep learning training,'' \emph{Proceedings of the National
  Academy of Sciences}, vol. 117, no.~40, pp. 24\,652--24\,663, 2020.

\bibitem{geiping2021stochastic}
J.~Geiping, M.~Goldblum, P.~Pope, M.~Moeller, and T.~Goldstein, ``Stochastic
  training is not necessary for generalization,'' in \emph{The 10th
  International Conference on Learning Representations (ICLR)}, 2022.

\bibitem{chiang2022loss}
P.-y. Chiang, R.~Ni, D.~Y. Miller, A.~Bansal, J.~Geiping, M.~Goldblum, and
  T.~Goldstein, ``Loss landscapes are all you need: Neural network
  generalization can be explained without the implicit bias of gradient
  descent,'' in \emph{The 11th International Conference on Learning
  Representations (ICLR)}, 2023.

\bibitem{shah2020pitfalls}
H.~Shah, K.~Tamuly, A.~Raghunathan, P.~Jain, and P.~Netrapalli, ``The pitfalls
  of simplicity bias in neural networks,'' \emph{Advances in Neural Information
  Processing Systems}, vol.~33, pp. 9573--9585, 2020.

\bibitem{hermann2023foundations}
K.~Hermann, H.~Mobahi, T.~FEL, and M.~C. Mozer, ``On the foundations of
  shortcut learning,'' in \emph{The 12th International Conference on Learning
  Representations (ICLR)}, 2024.

\bibitem{rahaman2019spectral}
N.~Rahaman, A.~Baratin, D.~Arpit, F.~Draxler, M.~Lin, F.~Hamprecht, Y.~Bengio,
  and A.~Courville, ``On the spectral bias of neural networks,'' in
  \emph{International conference on machine learning}.\hskip 1em plus 0.5em
  minus 0.4em\relax PMLR, 2019, pp. 5301--5310.

\bibitem{teney2024neural}
D.~Teney, A.~M. Nicolicioiu, V.~Hartmann, and E.~Abbasnejad, ``Neural redshift:
  Random networks are not random functions,'' in \emph{Proceedings of the
  IEEE/CVF Conference on Computer Vision and Pattern Recognition}, 2024, pp.
  4786--4796.

\bibitem{nguyen2015deep}
A.~Nguyen, J.~Yosinski, and J.~Clune, ``Deep neural networks are easily fooled:
  High confidence predictions for unrecognizable images,'' in \emph{Proceedings
  of the IEEE conference on computer vision and pattern recognition}, 2015, pp.
  427--436.

\bibitem{s23146378}
\BIBentryALTinterwordspacing
M.~Zhang, Y.~Chen, and C.~Qian, ``Fooling examples: Another intriguing property
  of neural networks,'' \emph{Sensors}, vol.~23, no.~14, 2023. [Online].
  Available: \url{https://www.mdpi.com/1424-8220/23/14/6378}
\BIBentrySTDinterwordspacing

\bibitem{KUMANO2023259}
S.~Kumano, H.~Kera, and T.~Yamasaki, ``Sparse fooling images: Fooling machine
  perception through unrecognizable images,'' \emph{Pattern Recognition
  Letters}, vol. 172, pp. 259--265, 2023.

\bibitem{goodfellow2014explaining}
I.~J. Goodfellow, J.~Shlens, and C.~Szegedy, ``Explaining and harnessing
  adversarial examples,'' in \emph{The 3rd International Conference on Learning
  Representations (ICLR)}, 2015.

\bibitem{zhang2019adversarial}
J.~Zhang and C.~Li, ``Adversarial examples: Opportunities and challenges,''
  \emph{IEEE transactions on neural networks and learning systems}, vol.~31,
  no.~7, pp. 2578--2593, 2019.

\bibitem{charikar2002similarity}
M.~S. Charikar, ``Similarity estimation techniques from rounding algorithms,''
  in \emph{Proceedings of the thiry-fourth annual ACM symposium on Theory of
  computing}, 2002, pp. 380--388.

\bibitem{kothapalli2022neural}
V.~Kothapalli, ``Neural collapse: A review on modelling principles and
  generalization,'' \emph{Transactions on Machine Learning Research}, 2023.

\bibitem{deng2012mnist}
L.~Deng, ``The mnist database of handwritten digit images for machine learning
  research [best of the web],'' \emph{IEEE Signal Processing Magazine},
  vol.~29, no.~6, pp. 141--142, 2012.

\bibitem{daugman2009iris}
\BIBentryALTinterwordspacing
J.~Daugman, ``Chapter 25 - how iris recognition works,'' in \emph{The Essential
  Guide to Image Processing}, A.~Bovik, Ed.\hskip 1em plus 0.5em minus
  0.4em\relax Boston: Academic Press, 2009, pp. 715--739. [Online]. Available:
  \url{https://www.sciencedirect.com/science/article/pii/B9780123744579000251}
\BIBentrySTDinterwordspacing

\bibitem{gaziv2023robustified}
G.~Gaziv, M.~J. Lee, and J.~J. DiCarlo, ``Robustified anns reveal wormholes
  between human category percepts,'' in \emph{37th Conference on Neural
  Information Processing Systems (NeurIPS)}, 2023.

\end{thebibliography}
\end{document}